\documentclass{article}

\usepackage{times}
\usepackage{graphicx} % more modern
\usepackage{subfigure} 

\usepackage{natbib}

\usepackage{algorithm}
\usepackage{algorithmic}

\usepackage{hyperref}
\usepackage{enumitem}
\usepackage{mathtools} 

\renewcommand{\algorithmicrequire}{\textbf{Input:}}
\renewcommand{\algorithmicensure}{\textbf{Output:}}

\usepackage[accepted]{icml2017}

\usepackage{amsmath,amsthm,amssymb,pdfsync}   
\usepackage{macros}

\usepackage{multirow}

\newcommand{\yingyu}[1]{}%{{\color{blue}\authnote{Yingyu}{{#1}}}}
\newcommand{\Anote}[1]{}%{{\color{blue}\authnote{Andrej}{{#1}}}}
\newcommand{\Ynote}[1]{}%{{\color{purple}\authnote{Yuanzhi}{{#1}}}}

\def\<{\langle}
\def\>{\rangle}

\newcommand{\hide}[1]{}

\newcommand{\veps} {\varepsilon}

\newtheorem{thm}{Theorem}

\newtheorem{prop}{Proposition}
\newtheorem{cor}[thm]{Corollary}
\newtheorem{lem}[thm]{Lemma}
\newtheorem{claim}[thm]{Claim}

\newtheorem{defn}{Definition}

\numberwithin{equation}{section}

\icmltitlerunning{Provable Alternating Gradient Descent for Non-negative Matrix Factorization with Strong Correlations}

\begin{document}

\twocolumn[
\icmltitle{Provable Alternating Gradient Descent for Non-negative Matrix Factorization with Strong Correlations}



% you can specify symbols, otherwise they are numbered in order
\icmlsetsymbol{equal}{*}

\begin{icmlauthorlist}
\icmlauthor{Yuanzhi Li}{to}
\icmlauthor{Yingyu Liang}{to}
\end{icmlauthorlist}

\icmlaffiliation{to}{Princeton University, Princeton, NJ, USA}

\icmlcorrespondingauthor{Yuanzhi Li}{yuanzhil@cs.princeton.edu}
\icmlcorrespondingauthor{Yingyu Liang}{yingyul@cs.princeton.edu}

\icmlkeywords{NMF, correlated topic model, gradient descent}

\vskip 0.3in
]

%\twocolumn[
%\icmltitle{Provable Alternating Gradient Descent for Non-negative Matrix Factorization with Strong Correlations}
%
%% It is OKAY to include author information, even for blind
%% submissions: the style file will automatically remove it for you
%% unless you've provided the [accepted] option to the icml2016
%% package.
%\icmlauthor{Yuanzhi Li, Yingyu Liang}{\{yuanzhil,yingyul\}@cs.princeton.edu}
%\icmladdress{Princeton University,
					   %35 Olden St, Princeton, NJ 08540 US}
%
%% You may provide any keywords that you 
%% find helpful for describing your paper; these are used to populate 
%% the "keywords" metadata in the PDF but will not be shown in the document
%\icmlkeywords{NMF, correlated topic model, gradient descent}
%
%\vskip 0.3in
%]

\printAffiliationsAndNotice{Authors listed in alphabetic order.}

\begin{abstract} 
Non-negative matrix factorization is a basic tool for decomposing data into the feature and weight matrices under non-negativity constraints, and in practice is often solved in the alternating minimization framework. 
However, it is unclear whether such algorithms can recover the ground-truth feature matrix when the weights for different features are highly correlated, which is common in applications.
This paper proposes a simple and natural alternating gradient descent based algorithm, and shows that with a mild initialization it provably recovers the ground-truth in the presence of strong correlations. 
In most interesting cases, the correlation can be in the same order as the highest possible. 
Our analysis also reveals its several favorable features including robustness to noise. 
We complement our theoretical results with empirical studies on semi-synthetic datasets, demonstrating its advantage over several popular methods in recovering the ground-truth.
\end{abstract} 

\section{Introduction}
\label{sec:intro}

% importance
Non-negative matrix factorization (NMF) is an important tool in data analysis and is widely used in image processing, text mining, and hyperspectral imaging (e.g.,~\cite{LeeSeu97,blei2003latent,yang2013overlapping}). 
Given a set of observations $\bY = \{ y^{(1)}, y^{(2)}, \ldots, y^{(n)}\}$, the goal of NMF is to find a feature matrix $\bA = \{ a_1, a_2, \ldots, a_D \}$ and a non-negative weight matrix $\bX = \{ x^{(1)}, x^{(2)}, \ldots, x^{(n)} \}$ such that $y^{(i)} \approx \bA x^{(i)}$ for any $i$, or $\bY \approx \bA \bX$ for short. The intuition of NMF is to write each data point as a \emph{non-negative} combination of the features. By doing so, one can avoid cancellation of different features and improve interpretability by thinking of each $x^{(i)}$ as a (unnormalized) probability distribution over the features. It is also observed empirically that the non-negativity constraint on the coefficients can lead to better features and improved downstream performance of the learned features. 

% gap between theory and practice 
Unlike the counterpart which factorizes $\bY \approx \bA \bX$ without assuming non-negativity of $\bX$, NMF is usually much harder to solve, and can even by NP-hard in the worse case~\cite{AroGeKanMoi12}. This explains why, despite all the practical success, NMF largely remains a mystery in theory.
Moreover, many of the theoretical results for NMF were based on very technical tools such has algebraic geometry (e.g., \cite{AroGeKanMoi12}) or tensor decomposition (e.g. \cite{anandkumar2012two}), which undermine their applicability in practice. Arguably, the most widely used algorithms for NMF use the \emph{alternative minimization scheme}: In each iteration, the algorithm alternatively keeps $\bA$ or $\bX$ as fixed and tries to minimize some distance between $\bY$ and $\bA \bX$. Algorithms in this framework, such as multiplicative update~\cite{LeeSeu01} and alternative non-negative least square~\cite{kim2008nonnegative}, usually perform well on real world data. However, alternative minimization algorithms are usually notoriously difficult to analyze. 
This problem is poorly understood, with only a few provable guarantees known~\cite{AwaRis15,li2016recovery}. Most importantly, these results are only for the case when the coordinates of the weights are from essentially independent distributions, while in practice they are known to be correlated, for example, in correlated topic models~\cite{blei2006correlated}. As far as we know, there exists no rigorous analysis of practical algorithms for the case with strong correlations. 

In this paper, we provide a theoretical analysis of a natural algorithm~\textsf{AND} (\textbf{A}lternative \textbf{N}on-negative gradient \textbf{D}escent) that belongs to the practical framework, and show that it probably recovers the ground-truth given a mild initialization. It works under general conditions on the feature matrix and the weights, in particular, allowing strong correlations. 
It also has multiple favorable features that are unique to its success.
We further complement our theoretical analysis by experiments on semi-synthetic data, demonstrating that the algorithm converges faster to the ground-truth than several existing practical algorithms, and providing positive support for some of the unique features of our algorithm. Our contributions are detailed below.

\subsection{Contributions}

In this paper, we assume a generative model of the data points, given the ground-truth feature matrix $\bAg$. In each round, we are given $y = \bAg x$,\footnote{We also consider the noisy case; see \ref{sec:afaafasf}.} where $x$ is sampled i.i.d. from some \emph{unknown} distribution $\mu$ and the goal is to recover the ground-truth feature matrix $\bAg$. We give an algorithm named \textsf{AND} that starts from a mild initialization matrix and provably converges to $\bAg$ in polynomial time. We also justify the convergence through a sequence of experiments. Our algorithm has the following favorable characteristics.

\subsubsection{Simple Gradient Descent Algorithm}
The algorithm \textsf{AND} runs in stages and keeps a working matrix $\bA^{(t)}$ in each stage. At the $t$-th iteration in a stage, after getting one sample $y$, it performs the following:
\begin{eqnarray*}
&\text{(Decode)}& z = \phi_{\alpha} \left((\bA^{(0)})^{\dagger} y\right),
\\
&\text{(Update)}& \bA^{(t + 1)} = \bA^{(t)} + \eta \left(y z ^{\top}  - \bA^{(t)} zz^{\top}  \right),
\end{eqnarray*}
where $\alpha$ is a threshold parameter, 
$$
\phi_{\alpha}(x) =  \left\{ \begin{array}{ll}
         x & \mbox{if $x \geq \alpha$},\\
        0 & \mbox{otherwise},\end{array} \right.
$$
$(\bA^{(0)})^{\dagger}$ is the Moore-Penrose pesudo-inverse of $\bA^{(0)}$,
and $\eta$ is the update step size. 
The decode step aims at recovering the corresponding weight for the data point, and the update step uses the decoded weight to update the feature matrix. The final working matrix at one stage will be used as the $\bA^{(0)}$ in the next stage. See Algorithm~\ref{alg:and} for the details.

 At a high level, our update step to the feature matrix can be thought of as a gradient descent version of alternative non-negative least square~\cite{kim2008nonnegative}, which at each iteration alternatively minimizes $L(\bA, \bZ) = \|\bY - \bA \bZ \|_F^2$ by fixing $\bA$ or $\bZ$. Our algorithm, instead of performing an complete minimization, performs only a stochastic \emph{gradient descent} step on the feature matrix. 
 To see this, consider one data point $y$ and consider minimizing $L(\bA, z) = \|y - \bA z \|_F^2$ with $z$ fixed. Then the gradient of $\bA$ is just $- \nabla L(\bA) = (y - \bA z) z^{\top}$, which is exactly the update of our feature matrix in each iteration. 
 
 As to the decode step, when $\alpha = 0$, our decoding can be regarded as a one-shot approach minimizing $\|\bY - \bA \bZ \|_F^2$ restricted to $\bZ \ge 0$. Indeed, if for example projected gradient descent is used to minimize $\|\bY - \bA \bZ \|_F^2$, then the projection step is exactly applying $\phi_{\alpha}$ to $\bZ$ with $\alpha=0$. A key ingredient of our algorithm is choosing  $\alpha$ to be larger than zero and then decreasing it, which allows us to outperform the standard algorithms.

  Perhaps worth noting, our decoding only uses $\bA^{(0)}$. Ideally, we would like to use $(\bA^{(t)})^{\dagger}$ as the decoding matrix in each iteration. However, such decoding method requires computing the pseudo-inverse of $\bA^{(t)}$ at every step, which is extremely slow. Instead, we divide the algorithm into stages and in each stage, we only use the starting matrix in the decoding, thus the pseudo-inverse only needs to be computed once per stage and can be used across all iterations inside. We can show that our algorithm converges in polylogarithmic many stages, thus gives us to a much better running time. These are made clear when we formally present the algorithm in Section~\ref{sec:algo} and the theorems in Section~\ref{sec:result_sim} and \ref{sec:result_gen}.
	
   \subsubsection{Handling strong correlations}
 
 The most notable property of \textsf{AND} is that it can provably deal with \emph{highly} correlated distribution $\mu$ on the weight $x$, meaning that the coordinates of $x$ can have very strong correlations with each other.\Ynote{need some good words for this} 
This is important since such correlated $x$ naturally shows up in practice. For example, when a document contains the topic ``machine learning'', \Ynote{machine learning} it is more likely to contain the topic ``computer science'' \Ynote{computer science} than ``geography''~\cite{blei2006correlated}. %\Ynote{geography} This has been noted in correlated topic models~\cite{blei2006correlated}.

Most of the previous theoretical approaches for analyzing alternating between decoding and encoding, such as~\cite{AwaRis15,li2016recovery,aroradictionary1}, require the coordinates of $x$ to be pairwise-independent, or almost pairwise-independent (meaning $\E_{\mu}[x_i x_j] \approx \E_{\mu}[x_i] \E_{\mu}[x_j]$). In this paper, we show that algorithm \textsf{AND} can recover $\bAg$ even when the coordinates are highly correlated. As one implication of our result, when the sparsity of $x$ is  $O(1)$ and each entry of $x$ is in $\{0, 1\}$, \textsf{AND}  can recover $\bAg$ even if each $\E_{\mu}[x_i x_j] = \Omega (\min\{\E_{\mu}[x_i], \E_{\mu}[x_j]\})$,  matching  (up to constant) the highest correlation possible. Moreover, we do not assume any prior knowledge about the distribution $\mu$, and the result also extends to general sparsities as well.

 \subsubsection{Pseudo-inverse decoding}
 One of the feature of our algorithm is to use Moore-Penrose pesudo-inverse in decoding. Inverse decoding was also used in~\cite{li2016recovery,aroradictionary1,arora2016provable}. However, their algorithms require  carefully finding an inverse such that certain norm is minimized, which is not as efficient as the vanilla Moore-Penrose pesudo-inverse. It was also observed in~\cite{arora2016provable} that Moore-Penrose pesudo-inverse works equally well in practice, but the experiment was done only when $\bA= \bAg$. In this paper, we show that Moore-Penrose pesudo-inverse  also works well when $\bA \not= \bAg$, both theoretically and empirically.

  \subsubsection{Thresholding at different $\alpha$} 
  Thresholding at a value $\alpha > 0$ is a common trick used in many algorithms. However, many of them still only consider a fixed $\alpha$ throughout the entire algorithm. Our contribution is a new method of thresholding that first sets $\alpha$ to be high, and gradually decreases $\alpha$ as the algorithm goes. Our analysis naturally provides the explicit rate at which we decrease $\alpha$, and shows that our algorithm, following this scheme, can provably converge to the ground-truth $\bAg$ in polynomial time. Moreover, we also provide experimental support for these choices. 
  
	\subsubsection{Robustness to noise} \label{sec:afaafasf}
	We further show that the algorithm is robust to noise. In particular, we consider the model $y = \bAg x + \zeta$, where $\zeta$ is the noise. The algorithm can tolerate a general family of noise with bounded moments; we present in the main body the result for a simplified case with Gaussian noise and provide the general result in the appendix. 
The algorithm can recover the ground-truth matrix up to a small blow-up factor times the noise level in \emph{each example}, when the ground-truth has a good condition number.
This robustness is also supported by our experiments.

%\textbf{Organization.} 
%After reviewing related work, we define the problem in Section~\ref{sec:problem} and describe the algorithm in Section~\ref{sec:algo}. To emphasize the key ideas, we first present the result and the proof sketch for a simplified yet still interesting case in Section~\ref{sec:result_sim}, and then present the results under much more general assumptions in Section~\ref{sec:result_gen}. The experimental results are presented in Section~\ref{sec:exp}. The complete proofs are provided in the appendix. 

\section{Related Work} \label{sec:related}

\textbf{Practical algorithms.}  Non-negative matrix factorization has a rich empirical history, starting with the practical algorithms of \cite{LeeSeu97,LeeSeu99,LeeSeu01}. 
It has been widely used in applications and there exist various methods for NMF, e.g.,~\cite{kim2008nonnegative,LeeSeu01,cichocki2007hierarchical,ding2013topic,ding2014efficient}. However, they do not have provable recovery guarantees. 

\textbf{Theoretical analysis.}
For theoretical analysis, \cite{AroGeKanMoi12} provided a fixed-parameter tractable algorithm for NMF using algebraic equations. They also provided matching hardness results: namely they show there is no algorithm running in time $(mW)^{o(D)}$ unless there is a sub-exponential running time algorithm for 3-SAT. 
\cite{AroGeKanMoi12} also studied NMF under separability assumptions about the features, and \cite{bhattacharyya2016nonnegative} studied NMF under related assumptions. 
The most related work is~\cite{li2016recovery}, which analyzed an alternating minimization type algorithm. However, the result only holds with strong assumptions about the distribution of the weight $x$, in particular, with the assumption that the coordinates of $x$ are independent. 

\textbf{Topic modeling.} Topic modeling is a popular generative model for text data~\cite{blei2003latent,blei2012probabilistic}. Usually, the model results in NMF type optimization problems with $\|x\|_1 = 1$, and a popular heuristic is \emph{variational inference}, which can be regarded as alternating minimization in KL-divergence. 
Recently, there is a line of theoretical work analyzing tensor decomposition~\cite{arora1,arora2,anandkumartopic} or combinatorial methods~\cite{AwaRis15}. These either need strong structural assumptions on the word-topic matrix $\bAg$, or need to know the distribution of the weight $x$, which is usually infeasible in applications.

\section{Problem and Definitions} \label{sec:problem}

We use $\| \bM \|_2$ to denote the 2-norm of a matrix $\bM$. $\| x\|_1$ is the 1-norm of a vector $x$. We use $[\bM]_i$ to denote the i-th row and $[\bM]^i$ to denote the $i$-th column of a matrix $\bM$. $\sigma_{\max}(\bM)(\sigma_{\min}(\bM))$ stands for the maximum (minimal) singular value of $\bM$, respectively. We consider a generative model for non-negative matrix factorization, where the data $y$ is generated from\footnote{Section \ref{sec:result_noise} considers the noisy case.}
\begin{align*}
y = \bAg x, \quad \bAg \in \mathbb{R}^{W \times D}
\end{align*}
where $\bAg$ is the ground-truth feature matrix, and $x$ is a non-negative random vector drawn from an unknown distribution $\mu$.
%and $\zeta$ is the noise.
The goal is to recover the ground-truth $\bAg$ from i.i.d. samples of the observation $y$. 

Since the general non-negative matrix factorization is NP-hard~\cite{AroGeKanMoi12},
some assumptions on the distribution of $x$ need to be made.
In this paper, we would like to allow distributions as general as possible, especially those with strong correlations.
Therefore, we introduce the following notion called $(r, k, m, \lambda)$-general correlation conditions (GCC) for the distribution of $x$.

%\begin{defn}[General Correlation Conditions, GCC]
%%A distribution of $x$ satisfies the $(r, k, m, \lambda)$-general correlation conditions, if the following are true.
%Let $\bDelta := \E[x x^{\top}]$, and let $\bDelta_d$ denote its diagonal part, and $\bDelta_o$ denote its off-diagonal part.  
%%where $\bDelta = \bDelta_d + \bDelta_o$ for diagonal matrix $\bDelta_d$ and off-diagonal matrix $\bDelta_o$.
%\begin{enumerate}
%\item $\| x \|_1 \le r$ and $\forall i \in [D]$, $x_i \in [0, 1]$.
%\item $ \bDelta_d \le \frac{2k}{D}$. 
%\item every entry $i, j$ of $\bDelta_o$ has $(\bDelta_o)_{i, j} \le \frac{m}{D^2}$. 
%\item $\bDelta \succeq  \frac{k}{D}\lambda \bI$.
%\end{enumerate}
%\end{defn}

\begin{defn}[General Correlation Conditions, GCC]
Let $\bDelta := \E[x x^{\top}]$ denote the second moment matrix.
\begin{enumerate}
\item $\| x \|_1 \le r$ and $x_i \in [0, 1], \forall i \in [D]$.
\item $ \bDelta_{i,i} \le \frac{2k}{D}, \forall i \in [D]$. 
\item $\bDelta_{i, j} \le \frac{m}{D^2}, \forall i \neq j \in [D]$. 
\item $\bDelta \succeq  \frac{k}{D}\lambda \bI$.
\end{enumerate}
\end{defn}

The first condition regularizes the sparsity of $x$.\footnote{Throughout this paper, the sparsity of $x$ refers to the $\ell_1$ norm, which is much weaker than the $\ell_0$ norm (the support sparsity). For example, in LDA, the $\ell_1$ norm of $x$ is always 1.} The second condition regularizes each coordinate of $x_i$ so that there is no $x_i$ being large too often. The third condition regularizes the maximum pairwise correlation between $x_i$ and $x_j$. The fourth  condition always holds for $\lambda = 0$ since $\E[x x^{\top}]$ is a PSD matrix. Later we will assume this condition holds for some $\lambda > 0$ to avoid degenerate cases. Note that we put the weight $k/D$ before $\lambda$ such that $\lambda$ defined in this way will be a positive constant in many interesting examples discussed below. %\footnote{In particular, $\lambda \in [0, 2]$ following this definition.} 

%The second condition assumes that each $x_i$  roughly has the same weight. This assumption appears to be to restricted, however, since we have made no assumption on $\bAg$, we can always instead look at a generative model $y = (\bAg \bSigma)(\bSigma^{-1} x)$ for a diagonal matrix $\bSigma$, and by defining $(\bAg \bSigma)$ to be the new ground-truth and $(\bSigma^{-1} x)$ to be the new variables, we can always find one $\bSigma$ such that the new $x$ satisfies this condition

To get a sense of what are the ranges of $k, m,$ and $\lambda$ given sparsity $r$, we consider the following most commonly studied non-negative random variables.

\begin{prop}[Examples of GCC]  \label{pro:gcc}
~\vspace{-2mm}
\begin{enumerate}
\item If $x$ is  chosen uniformly over $s$-sparse random vectors with $\{0, 1\}$ entries, then $k = r = s$, $m = s^2$ and $\lambda = 1- \frac{1}{s}$.
\item  If $x$ is uniformly chosen from Dirichlet distribution with parameter $\alpha_i = \frac{s}{D}$, then $r = k = 1$ and $m = \frac{1}{sD}$  with $\lambda = 1 - \frac{1}{s}$.
\end{enumerate}
\end{prop}

For these examples, the result in this paper shows that we can recover $\bAg$ for aforementioned random variables $x$ as long as $s = O(D^{1/6})$. 
%However, our algorithm can also be applied to a significantly more general set of random variables $x$. 
In general, there is a wide range of  parameters $(r, k, m, \lambda)$ such that learning $\bAg$ is doable with polynomially many samples of $y$ and in polynomial time. 

However, just the GCC condition is not enough for recovering $\bAg$.
We will also need a mild initialization. 

\begin{defn}[$\ell$-initialization]
The initial matrix $\bA_0$ satisfies for some $\ell \in [0, 1)$,
\begin{enumerate}
\item $\bA_0 = \bAg (\bSigma + \bE)$, for some diagonal matrix $\bSigma$ and off-diagonal matrix $\bE$. 
\item $\|\bE\|_2 \le \ell$, $\|\bSigma - \bI\|_2 \le \frac{1}{4} $.
\end{enumerate}
\end{defn}

The condition means that the initialization is not too far away from the ground-truth $\bAg$. For any $i \in [D]$, the $i$-th column $[\bA_0]^i =  \bSigma_{i,i} [\bAg]^i + \sum_{j \neq i} \bE_{j,i} [\bAg]^j$. So the condition means that each feature $[\bA_0]^i$ has a large fraction of the ground-truth feature $[\bAg]^i$ and a small fraction of the other features. $\bSigma$ can be regarded as the magnitude of the component from the ground-truth in the initialization, while $\bE$ can be regarded as the magnitude of the error terms. In particular, when $\bSigma = \bI$ and $\bE = 0$, we have $\bA_0 = \bAg$. The initialization allows $\bSigma$ to be a constant away from $\bI$, and the error term $\bE$ to be $\ell$ (in our theorems $\ell$ can be as large as a constant). 

In practice, such an initialization is typically achieved by setting the columns of $\bA_0$ to reasonable ``pure'' data points that contain one major feature and a small fraction of some others (e.g. \cite{ldac,AwaRis15}). 

\section{Algorithm} \label{sec:algo}

\begin{algorithm}[t!]
\caption{Alternating Non-negative gradient Descent (\textsf{AND}) \label{alg:and}}
\begin{algorithmic}[1]
\REQUIRE Threshold values $\{ \alpha_0, \alpha_1, \ldots, \alpha_s \}$, $T$, $\bA_0$
\STATE $\bA^{(0)} \leftarrow \bA_0$
\FOR{$j = 0, 1, \ldots, s$}
\FOR{$t = 0, 1, \ldots, T$}
\STATE On getting sample $y^{(t)}$, do: 
\STATE  $z^{(t)} \leftarrow \phi_{\alpha_j} \left((\bA^{(0)} )^{\dagger} y^{(t)} \right)$
\STATE $\bA^{(t + 1)} \leftarrow \bA^{(t)} + \eta \left(y^{(t)}  - \bA^{(t)} z^{(t)} \right) (z^{(t)} )^{\top} $
\ENDFOR
\STATE $\bA^{(0)} \leftarrow \bA^{(T + 1)}$
\ENDFOR
\ENSURE $\bA \leftarrow \bA^{(T + 1)}$
\end{algorithmic}
\end{algorithm}

The algorithm is formally describe in Algorithm~\ref{alg:and}.
It runs in $s$ stages, and in the $j$-th stage, uses the same threshold $\alpha_j$ and the same matrix $\bA^{(0)}$ for decoding, 
where $\bA^{(0)}$ is either the input initialization matrix or the working matrix obtained at the end of the last stage. Each stage consists of $T$ iterations, and each iteration decodes one data point and uses the decoded result to update the working matrix. It can use a batch of data points instead of one data point, and our analysis still holds.

By running in stages, we save most of the cost of computing $(\bA^{(0)})^{\dagger}$, as our results show that only polylogarithmic stages are needed. 
For the simple case where $x \in \{0,1\}^D$, the algorithm can use the same threshold value $\alpha = 1/4$ for all stages (see Theorem~\ref{thm:main_binary}), while for the general case, it needs decreasing threshold values across the stages (see Theorem~\ref{thm:main}). Our analysis provides the hint for setting the threshold; see the discussion after Theorem~\ref{thm:main}, and Section~\ref{sec:exp} for how to set the threshold in practice. 

\section{Result for A Simplified Case} \label{sec:result_sim}

In this section, we consider the following simplified case:
\begin{align} \label{model:sim}
 y = \bAg x, ~~x \in \{0, 1\}^D.
\end{align}
That is, the weight coordinates $x_i$'s are binary. 

\begin{thm}[Main, binary] \label{thm:main_binary}
For the generative model~(\ref{model:sim}), there exists $\ell = \Omega(1)$ such that for every $(r, k, m, \lambda)$-GCC $x$ and every $\epsilon > 0$,  Algorithm \textsf{AND} with $T = \poly(D, \frac{1}{\epsilon}), \eta = \frac{1}{ \poly(D, \frac{1}{\epsilon})}$, $\{\alpha_i \}_{i = 1}^s = \{ \frac{1}{4} \}_{i = 1}^s$ for $s = \textsf{polylog}(D, \frac{1}{\epsilon})$ and  an $\ell$ initialization matrix $\bA_0$, outputs a matrix $\bA$ such that there exists a diagonal matrix $\bSigma \succeq \frac{1}{2} \bI$ with $\| \bA - \bAg \bSigma \|_2 \le \epsilon$ using $\poly(D, \frac{1}{\epsilon})$ samples  and iterations, as long as 
$$m = O\left( \frac{k D \lambda^4}{r^5} \right).$$
\end{thm}

Therefore, our algorithm recovers the ground-truth $\bAg$ up to scaling. The scaling in unavoidable since there is no assumption on $\bAg$, so we cannot, for example, distinguish $\bAg$ from $2 \bAg$. Indeed, if we in addition assume each column of $\bAg$ has norm $1$ as typical in applications, then we can recover $\bAg$ directly. In particular, by normalizing each column of $\bA$ to have norm 1, we can guarantee that 
$
 \| \bA - \bAg  \|_2 \le O(\epsilon).
$

\Ynote{add one sentence saying that in most of the interesting applications $k, r, \lambda$ are constants, so we can recover $\bAg$ even when $m = O(D)$}

In many interesting applications (for example, those in Proposition~\ref{pro:gcc}), $k, r, \lambda$ are constants. The theorem implies that the algorithm can recover $\bAg$ even when $m = O(D)$. In this case, $\E_{\mu}[x_i x_j]$ can be as large as $O(1/D)$, the same order as $\min\{\E_{\mu}[x_i], \E_{\mu}[x_j]\}$, which is the highest possible correlation.

%\paragraph{Exact recovery from finite samples.}
%Although we stated our theorem for $\poly(D, \frac{1}{\epsilon})$ samples, our algorithm can actually converge to the ground-truth ($\epsilon \to 0$) with a finite number of samples. In particular, given finitely many samples $y_1, \ldots, y_S$ coming from $y_i = \bAg x_i$, we can regard the \emph{uniform distribution} $\mu_u$ over $\{ x_i \}_{i = 1}^S$ as our underlying distribution $\mu$, and thus each time we just need to sample from $\mu_u$. As long as $\mu_u$ satisfies our GCC condition, our algorithm can also recover $\bAg$ exactly. In contrast, most of the existing algorithms (e.g.,~\cite{AwaRis15,li2016recovery}) truly requires infinite samples to make the error go to zero. 

\subsection{Intuition}

The intuition comes from assuming that we have the ``correct decoding'', that is, suppose magically for every $y^{(t)}$, our decoding $z^{(t)} = \phi_{\alpha_j} (\bA^{\dagger} y^{(t)}) = x^{(t)}$. Here and in this subsection, $\bA$ is a shorthand for $\bA^{(0)}$. The gradient descent is then $ \bA^{(t + 1)}= \bA^{(t)} + \eta (y^{(t)} - \bA^{(t)} x^{(t)}) (x^{(t)})^{\top}$. Subtracting $\bAg$ on both side, we will get $$ (\bA^{(t + 1)} - \bAg) = (\bA^{(t)} - \bAg)(\bI - \eta x^{(t)} (x^{(t)})^{\top} )$$

Since $x^{(t)} (x^{(t)})^{\top} $ is positive semidefinite, as long as $\E[x^{(t)} (x^{(t)})^{\top}] \succ 0$ and $\eta$ is sufficiently small, $\bA^{(t)}$ will converge to $\bAg$ eventually.

However, this simple argument does not work when $\bA \not= \bAg$ and thus we do not have the correct decoding. For example, if we just let the decoding be $\tilde{z}^{(t)} = \bA^{\dagger} y^{(t)}$, we will have $y^{(t)} - \bA \tilde{z}^{(t)} =  y^{(t)}  - \bA^{\dagger} \bA y^{(t)} = (\bI -  \bA^{\dagger} \bA) \bAg x^{(t)}$. Thus, using this decoding, the algorithm can never make \emph{any} progress once $\bA$ and $\bAg$ are in the same subspace.

The most important piece of our proof is to show that after \emph{thresholding}, $z^{(t)} = \phi_{\alpha} (\bA^{\dagger} y^{(t)})$ is much closer to $x^{(t)}$ than $\tilde{z}^{(t)}$.  Since $\bA$ and $\bAg$ are in the same subspace, inspired by \cite{li2016recovery} we can write $\bAg$ as $\bA (\bSigma + \bE)$ for a diagonal matrix $\bSigma$ and an off-diagonal matrix $\bE$, and thus the decoding becomes $z^{(t)} = \phi_{\alpha} (\bSigma x^{(t)} + \bE x^{(t)})$. Let us focus on one coordinate of $z^{(t)} $, that is, $z^{(t)}_i = \phi_{\alpha} (\bSigma_{i, i} x^{(t)}_i + \bE_i x^{(t)})$, where $\bE_i$ is the $i$-th row of $\bE_i$. The term $\bSigma_{i, i} x^{(t)}_i $ is a nice term since it is just a rescaling of $x^{(t)}_i$, while $\bE_i x^{(t)}$ mixes different coordinates of $x^{(t)}$. For simplicity, we just assume for now that $x_i^{(t)} \in \{0, 1\}$ and $\bSigma_{i, i} = 1$.  In our proof, we will show that the threshold will \emph{remove} a large fraction of $\bE_i x^{(t)}$ when $x^{(t)}_i = 0$, and keep a large fraction of $\bSigma_{i, i} x^{(t)}_i $ when $x^{(t)}_i = 1$. Thus, our decoding is much more accurate than without thresholding. To show this, we maintain a crucial property that for our decoding matrix, we always have $\| \bE_i\|_2 = O(1)$. Assuming this, we first consider two extreme cases of $\bE_i$.
\begin{enumerate}
\item Ultra dense: all coordinates of $\bE_i$ are in the order of $\frac{1}{\sqrt{d}}$. Since the sparsity of $x^{(t)}$ is $r$, as long as $r = o(\sqrt{d} ) \alpha$, $\bE_i x^{(t)}$ will not pass $\alpha$ and thus $z^{(t)}_i $ will be decoded to zero when $x^{(t)}_i  = 0$.
\item Ultra sparse: $\bE_i$ only has few coordinate equal to $\Omega(1)$ and the rest are zero. Unless $x^{(t)}$ has those exact coordinates equal to $1$ (which happens not so often), then $z^{(t)}_i $ will still be zero when $x^{(t)}_i  = 0$.
\end{enumerate}
Of course, the real $\bE_i$ can be anywhere in between these two extremes, and thus we need more delicate decoding lemmas, as shown in the complete proof.

Furthermore, more complication arises when each $x^{(t)}_i$ is not just in $\{0,1\}$ but can take fractional values. To handle this case, we will set our threshold $\alpha$ to be large at the beginning and then keep shrinking after each stage. The intuition here is that we first decode the coordinates that we are most confident in, so we do not decode $z^{(t)}_i $ to be non-zero when $x^{(t)}_i  = 0$. Thus, we will still be able to remove a large fraction of error caused by $\bE_i x^{(t)}$. However, by setting the threshold $\alpha$ so high, we may introduce more errors to  the nice term $\bSigma_{i, i} x^{(t)}_i $ as well, since $\bSigma_{i, i} x^{(t)}_i $ might not be larger than $\alpha$ when $ x^{(t)}_i  \not= 0$. Our main contribution is to show that there is a nice trade-off between the errors in $\bE_i$ terms and those in $\bSigma_{i, i}$ terms such that as we gradually decreases $\alpha$, the algorithm can converge to the ground-truth. 

%\yingyu{maybe add empirical verification about this}

\subsection{Proof Sketch}

%\yingyu{Put the main lemma here.}

For simplicity, we only focus on one stage and the expected update.  The expected update of $\bA^{(t)}$ is given by
$$\bA^{(t + 1)} = \bA^{(t)} + \eta (\E[y z^{\top}]  - \bA^{(t)} \E[z z^{\top}]).$$

Let us write $\bA^{(0)} = \bAg (\bSigma_0 + \bE_0)$ where $\bSigma_0$ is diagonal and $\bE_0$ is off-diagonal. Then the decoding is given by
$$z = \phi_{\alpha} ((\bA^{(0)})^{\dagger} y ) = \phi_{\alpha} ( (\bSigma_0 + \bE_0)^{-1 } x).$$
Let $\bSigma, \bE$ be the diagonal part and the off-diagonal part of $ (\bSigma_0 + \bE_0)^{-1 }$. 

The key lemma for decoding says that under suitable conditions, $z$ will be close to $\bSigma x $ in the following sense.

\begin{lem}[Decoding, informal]
Suppose $\bE$ is small and $\bSigma \approx \bI$. Then with a proper threshold value $\alpha$, we have  
$$ \E[ \bSigma xx^\top] \approx \E[z x^{\top}], ~~\E[\bSigma xz^{\top}] \approx \E[z z^{\top}].$$
\end{lem}

Now, let us write $\bA^{(t)}  = \bAg (\bSigma_t + \bE_t)$. Then applying the above decoding lemma, the expected update of $\bSigma_t + \bE_t$ is 
$$\bSigma_{t + 1}+ \bE_{t + 1} = ( \bSigma_t + \bE_t )( \bI   - \bSigma \bDelta \bSigma ) + \bSigma^{-1} (\bSigma \bDelta \bSigma)  + \bR_t$$
where $\bDelta = \E[x x^{\top}]$ and $\bR_t$ is a small error term. 

Our second key lemma is about this update.
\begin{lem}[Update, informal]\label{lem:main_simple_info}
Suppose the update rule is
$$\bSigma_{t + 1} + \bE_{t + 1}  =(  \bSigma_{t } + \bE_{t }  )(1 - \eta \bLambda) + \eta \bQ \bLambda + \eta \bR_t$$
for some PSD matrix $\bLambda$ and $\|\bR_t \|_2\le C''$. Then
\begin{align*}
\|\bSigma_t + \bE_t - \bQ \|_2 & \le \|  \bSigma_{0} + \bE_{0 } - \bQ   \|_2 (1 -\eta \lambda_{\min}(\bLambda))^t \\
 &  +  \frac{C''}{\lambda_{\min}(\bLambda)}. 
\end{align*}
\end{lem}

Applying this on our update rule with $\bQ = \bSigma^{-1}$ and $\bLambda = \bSigma \bDelta \bSigma$, we  know that when the  error term is sufficiently small, we can make progress on $\|\bSigma_t + \bE_t - \bSigma^{-1} \|_2$.
Furthermore, by using the fact that
$\bSigma_0  \approx \bI$ and $\bE_0$ is small, and the fact that $\bSigma$ is the diagonal part of $ (\bSigma_0 + \bE_0)^{-1 }$, we can show that after sufficiently many iterations, $\|\bSigma_t - \bI\|_2$ blows up slightly, while $\|\bE_t\|_2$ is reduced significantly. 
Repeating this for multiple stages completes the proof.

We note that most technical details are hidden, especially for the proofs of the decoding lemma, which need to show that the error term $\bR_t$ is small. This crucially relies on the choice of $\alpha$, and relies on bounding the effect of the correlation. These then give the setting of $\alpha$ and the bound on the parameter $m$ in the final theorem.  
\yingyu{say more; esp. how to set $\alpha$. Then say sth about $\alpha$ in the next section}

\section{More General Results} \label{sec:result_gen}

\subsection{Result for General $x$} \label{sec:result_gen_x} 

This subsection considers the general case where 
$x\in [0,1]^D$. 
Then the GCC condition is not enough for recovery, even for $k, r, m = O(1)$ and $\lambda = \Omega (1)$. For example, GCC does not rule out the case that $x$ is drawn uniformly over $(r-1)$-sparse random vectors with $\{\frac{1}{D},  1\}$ entries, when one cannot recover even a reasonable approximation of $\bAg$ since a common vector $\frac{1}{D} \sum_i [\bAg]^i$ shows up in all the samples. This example shows that the difficulty arises if each $x_i$ constantly shows up with a small value. To avoid this, a general and natural way is to assume that each $x_i$, once being non-zero, has to take a large value with sufficient probability. This is formalized as follows.

\begin{defn}[Decay condition]
A distribution of $x$ satisfies the order-$q$ decay condition for some constant $q \geq 1$, if for all $i \in [D]$, $x_i$ satisfies that for every $\alpha > 0$, 
$$\Pr[x_i \leq \alpha \mid x_i \not = 0] \le \alpha^q.$$
\end{defn}

When $q = 1$, each $x_i$, once being non-zero, is uniformly distributed in the interval $[0, 1]$. When $q$ gets larger, each $x_i$, once being non-zero, will be more likely to take larger values. We will show that our algorithm has a better guarantee for larger $q$. In the extreme case when $q = \infty$, $x_i$ will only take $\{0, 1\}$ values, which reduces to the binary case.
%To illustrate the idea, we also present here the following simple theorem in this case. 

In this paper, we show that this simple decay condition, combined with the GCC conditions and an initialization with constant error, is sufficient for recovering $\bAg$. 

\begin{thm}[Main] \label{thm:main}
 There exists $\ell = \Omega(1)$ such that for every $(r, k, m, \lambda)$-GCC $x$ satisfying the order-$q$ condition, every $\epsilon > 0$, there exists $T, \eta$ and a sequence of $\{ \alpha_i \}$ \footnote{In fact, we will make the choice explicit in the proof.} such that Algorithm \textsf{AND},  with $\ell$-initialization matrix $\bA_0$, outputs a matrix $\bA$ such that there exists a diagonal matrix $\bSigma \succeq \frac{1}{2} \bI$ with  $\| \bA - \bAg \bSigma \|_2 \le \epsilon$ with $\poly(D, \frac{1}{\epsilon})$ samples and iterations, as long as
 $$m = O \left(\frac{k D^{1 - \frac{1}{q}} \lambda^{4 + \frac{4}{q }}}{r^{5 + \frac{6}{q + 1}}} \right).$$
 \end{thm}

As mentioned, in many interesting applications, $k = r = \lambda = \Theta(1)$, where our algorithm can recover $\bAg$ as long as $m = O(D^{1 - \frac{1}{q + 1}})$. This means $\E_{\mu}[x_i x_j] = O(D^{- 1 - \frac{1}{q + 1}})$, a factor of $D^{- \frac{1}{q + 1}}$ away from the highest possible correlation $\min\{\E_{\mu}[x_i], \E_{\mu}[x_j]\} = O(1/D)$. Then, the larger $q$, the higher correlation it can tolerate.
As $q$ goes to infinity, we recover the result for the case $x \in \{0,1\}^D$, allowing the highest order correlation.

The analysis also shows that the decoding threshold should be $\alpha = \left( \frac{\lambda \|\bE_0\|_2 }{r} \right)^{\frac{2}{q + 1}}$ where $\bE_0$ is the error matrix at the beginning of the stage. Since the error decreases exponentially with stages, this suggests to decrease $\alpha$ exponentially with stages. This is crucial for \textsf{AND} to recover the ground-truth; see Section~\ref{sec:exp} for the experimental results.

%We also note that the powers in our theorem are not optimal in this statement, as we want to present the cleanest result. See the appendix for a better bound.

\vspace{-2mm}
\subsection{Robustness to Noise} \label{sec:result_noise}

We now consider the case when the data is generated from 
$ %\begin{align*}
 y = \bAg x + \zeta,
$ %\end{align*}
where $\zeta$ is the noise. 
For the sake of demonstration, we will just focus on the case when $x_i \in \{0, 1\}$ and $\zeta$ is random Gaussian noise $\zeta \sim \gamma \mathcal{N}\left(0, \frac{1}{W}\bI \right)$. \footnote{we make this scaling so $\| \zeta\|_2 \approx \gamma$.} A more general theorem can be found in the appendix.

\begin{defn}[$(\ell, \rho)$-initialization]
The initial matrix $\bA_0$ satisfies for some $\ell , \rho \in [0, 1)$,
\begin{enumerate}
\item $\bA_0 = \bAg (\bSigma + \bE) + \bNoise$, for some diagonal matrix $\bSigma$ and off-diagonal matrix $\bE$. 
\item $\|\bE\|_2 \le \ell$, $\|\bSigma - \bI\|_2 \le \frac{1}{4} $, $\|\bNoise \|_2 \le \rho$. 
\end{enumerate}

\end{defn}

 \begin{thm}[Noise, binary] \label{thm:main_noise_simple}
Suppose each $x_i \in \{0, 1\}$. There exists $\ell = \Omega(1)$ such that for every $(r, k, m, \lambda)$-GCC $x$, every $\epsilon > 0$,  Algorithm \textsf{AND} with $T = \poly(D, \frac{1}{\epsilon}), \eta = \frac{1}{ \poly(D, \frac{1}{\epsilon})}$, $\{\alpha_i\}_{i=1}^s = \{\frac{1}{4}\}_{i=1}^4$ and an $(\ell, \rho)$-initialization $\bA_0$ for $\rho = O( \sigma_{\min}(\bAg))$, outputs $\bA$ such that there exists a diagonal matrix $\bSigma \succeq \frac{1}{2} \bI$ with  $$\| \bA - \bAg \bSigma \|_2 \le O\left(\epsilon +  r \frac{\sigma_{\max} (\bAg)}{\sigma_{\min} (\bAg) \lambda} \gamma\right)$$ using $\poly(D, \frac{1}{\epsilon})$ iterations, as long as
$m = O\left( \frac{k D \lambda^4}{r^5} \right).$
\end{thm}

The theorem implies that the algorithm can recover the ground-truth up to $ r \frac{\sigma_{\max} (\bAg)}{\sigma_{\min} (\bAg) \lambda}$ times $\gamma$, the noise level in each sample. 
Although stated here for Gaussian noise for simplicity, the analysis applies to a much larger class of noises, including adversarial ones. In particular, we only need to the noise $\zeta$ have sufficiently bounded $\| \E[\zeta \zeta^\top] \|_2$; see the appendix for the details. 
For the special case of Gaussian noise, by exploiting its properties, it is possible to improve the error term with a more careful calculation, though not done here.
%up to $\frac{\sigma_{\max} (\bAg)}{\sigma_{\text{avg}} (\bAg)}$ times the noise level with a more refined calculation, where $\sigma_{\text{avg}}$ is the average of the singular values. This is left for future work.

%\Ynote{We admit that $ \frac{\sigma_{\max} (\bAg)}{\sigma_{\min} (\bAg)}$ is rather large, however, in fact our analysis is made for a much larger class of noises, including fully adversarial ones and present in first 8 pages only Gaussian cases for simplicity. We didn't use many properties of noise being Gaussian, and we believe it can be improved to $\frac{\sigma_{\max} (\bAg)}{\sigma_{\text{avg}} (\bAg)}$ with a more refined calculation, where $\sigma_{\text{avg}}$ is the average of singular values. We highlight the potential improvement place in the proof, and this can be an interesting future direction.}

%\yingyu{can put this in the appendix about future directions.}

\vspace{-2mm}
\section{Experiments} \label{sec:exp}

To demonstrate the advantage of \textsf{AND}, we complement the theoretical analysis with empirical study on semi-synthetic datasets, where we have ground-truth feature matrices and can thus verify the convergence. We then provide support for the benefit of using decreasing thresholds, and test its robustness to noise. 
In the appendix, we further test its robust to initialization and sparsity of $x$, and provide qualitative results in some real world applications. \footnote{The code is public on \url{https://github.com/PrincetonML/AND4NMF}.}

\vspace{-2mm}
\paragraph{Setup.}
Our work focuses on convergence of the solution to the ground-truth feature matrix. 
However, real-world datasets in general do not have ground-truth. So we construct semi-synthetic datasets in topic modeling: first take the word-topic matrix learned by some topic modeling method as the ground-truth $\bAg$, and then draw $x$ from some specific distribution $\mu$. 
For fair comparison, we use one not learned by any algorithm evaluated here. In particular, we used the matrix with 100 topics computed by the algorithm in~\cite{arora2} on the NIPS papers dataset (about 1500 documents, average length about 1000). Based on this we build two semi-synthetic datasets:
\begin{enumerate}
	\item[1.] \textsf{DIR}. Construct a $100 \times 5000$ matrix $\bX$, whose columns are from a Dirichlet distribution with parameters $(0.05, 0.05, \ldots, 0.05)$. Then the dataset is 
	$ \bY = \bAg \bX$. 	
	%This leads to a dataset with relatively weak correlation. 
	\item[2.] \textsf{CTM}. The matrix $\bX$ is of the same size as above, while each column is drawn from the logistic normal prior in the correlated topic model~\cite{blei2006correlated}. This leads to a dataset with strong correlations. 
\end{enumerate}

Note that the word-topic matrix is non-negative. While some competitor algorithms require a non-negative feature matrix, \textsf{AND} does not need such a condition. To demonstrate this, we generate the following synthetic data:  

\begin{enumerate} 
	\item[3.] \textsf{NEG}. The entries of the matrix $\bAg$ are i.i.d. samples from the uniform distribution on $[-0.5, 0.5)$. The matrix $\bX$ is the same as in \textsf{CTM}. 
\end{enumerate}

Finally, the following dataset is for testing the robustness of \textsf{AND} to the noise:
\begin{enumerate} 
	\item[4.] \textsf{NOISE}. $\bAg$ and $\bX$ are the same as in \textsf{CTM}, but $\bY = \bAg \bX + \mathbf{N}$ where $\mathbf{N}$ is the noise matrix with columns drawn from $\gamma \mathcal{N}\left(0, \frac{1}{W}\bI \right)$ with the noise level $\gamma$. 
\end{enumerate}

\vspace{-2mm}
\paragraph{Competitors.} We compare the algorithm \textsf{AND} to the following popular methods: Alternating Non-negative Least Square (\textsf{ANLS}~\cite{kim2008nonnegative}), multiplicative update (\textsf{MU}~\cite{LeeSeu01}), 
\textsf{LDA} (online version~\cite{hoffman2010online}),\footnote{We use the implementation in the sklearn package (\url{http://scikit-learn.org/})
	} and Hierarchical Alternating Least Square (\textsf{HALS}~\cite{cichocki2007hierarchical}).
%\begin{enumerate}
	%\item \textsf{ANLS}~\cite{kim2008nonnegative}.
	%\item \textsf{MU}~\cite{LeeSeu01}.
	%\item \textsf{LDA} (online version~\cite{hoffman2010online}).\footnote{We use the implementation in the sklearn package (\url{http://scikit-learn.org/}).
	%}
	%\item \text{HALS}~\cite{cichocki2007hierarchical}.
%\end{enumerate}

\newcommand{\scalef}{0.33}
\begin{figure*}[!h]
\centering
\subfigure[on \textsf{DIR} dataset]{\includegraphics[width=\scalef\linewidth]{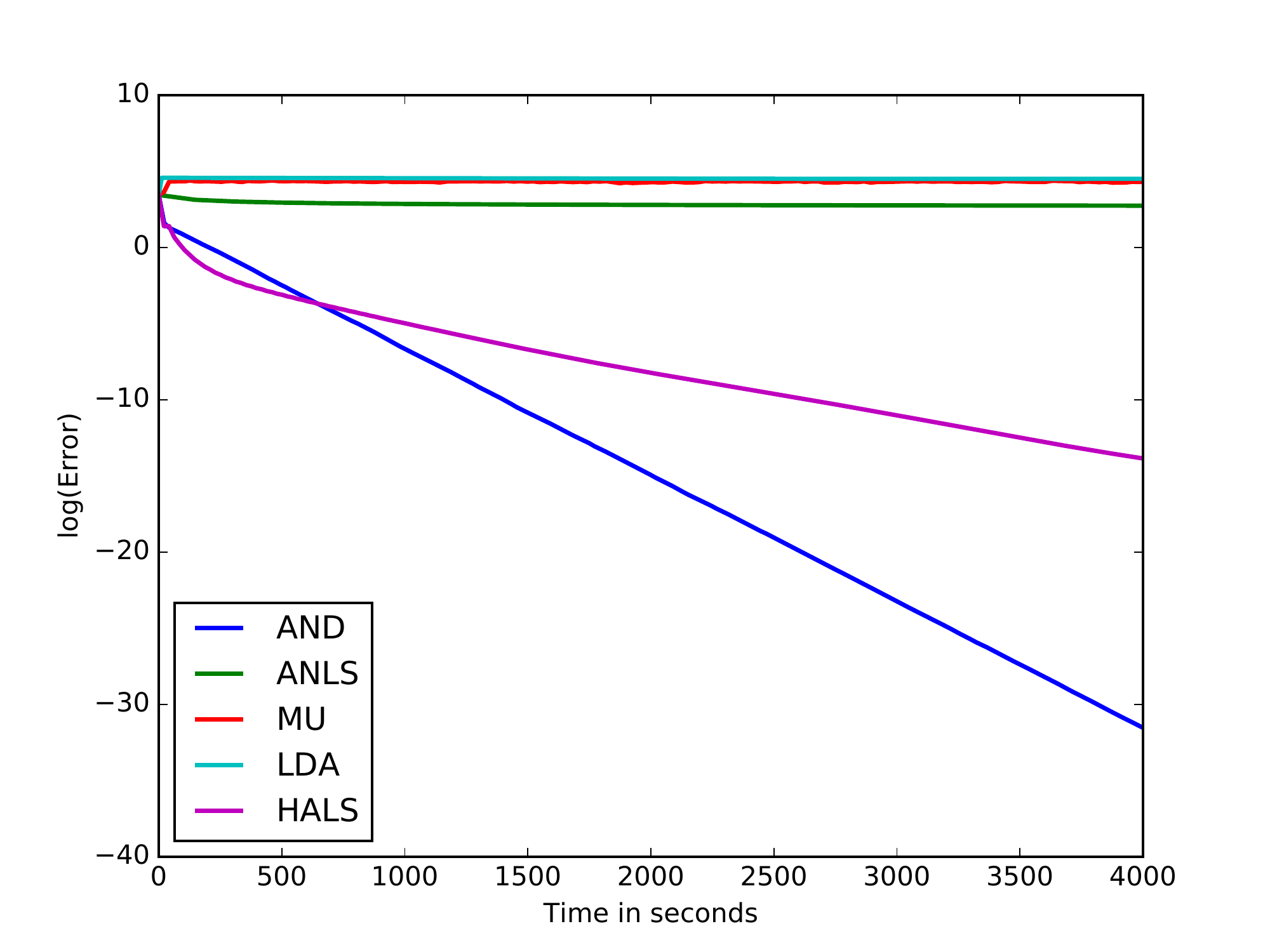}}
\subfigure[on \textsf{CTM} dataset]{\includegraphics[width=\scalef\linewidth]{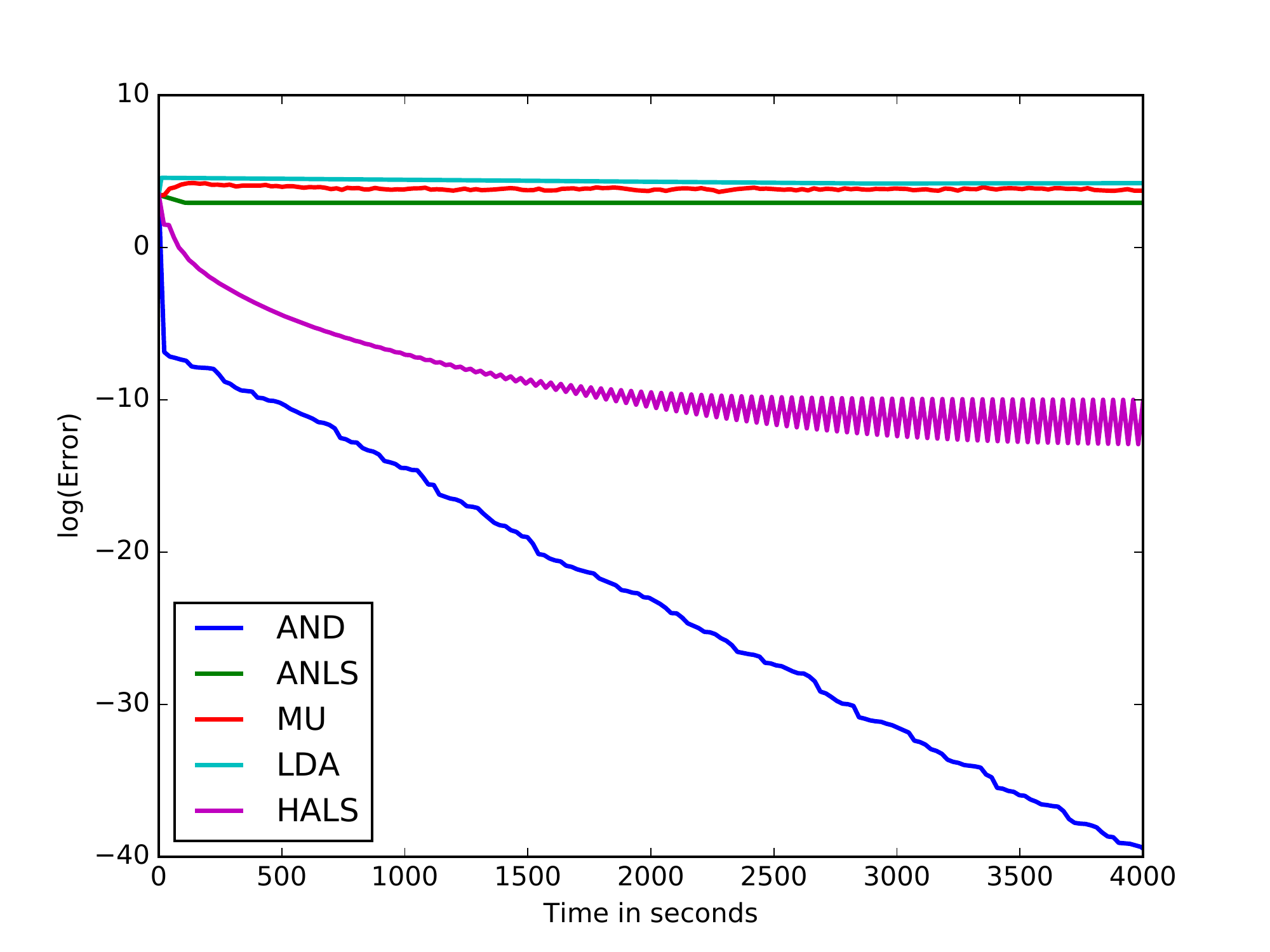}}
\subfigure[on \textsf{NEG} dataset]{\includegraphics[width=\scalef\linewidth]{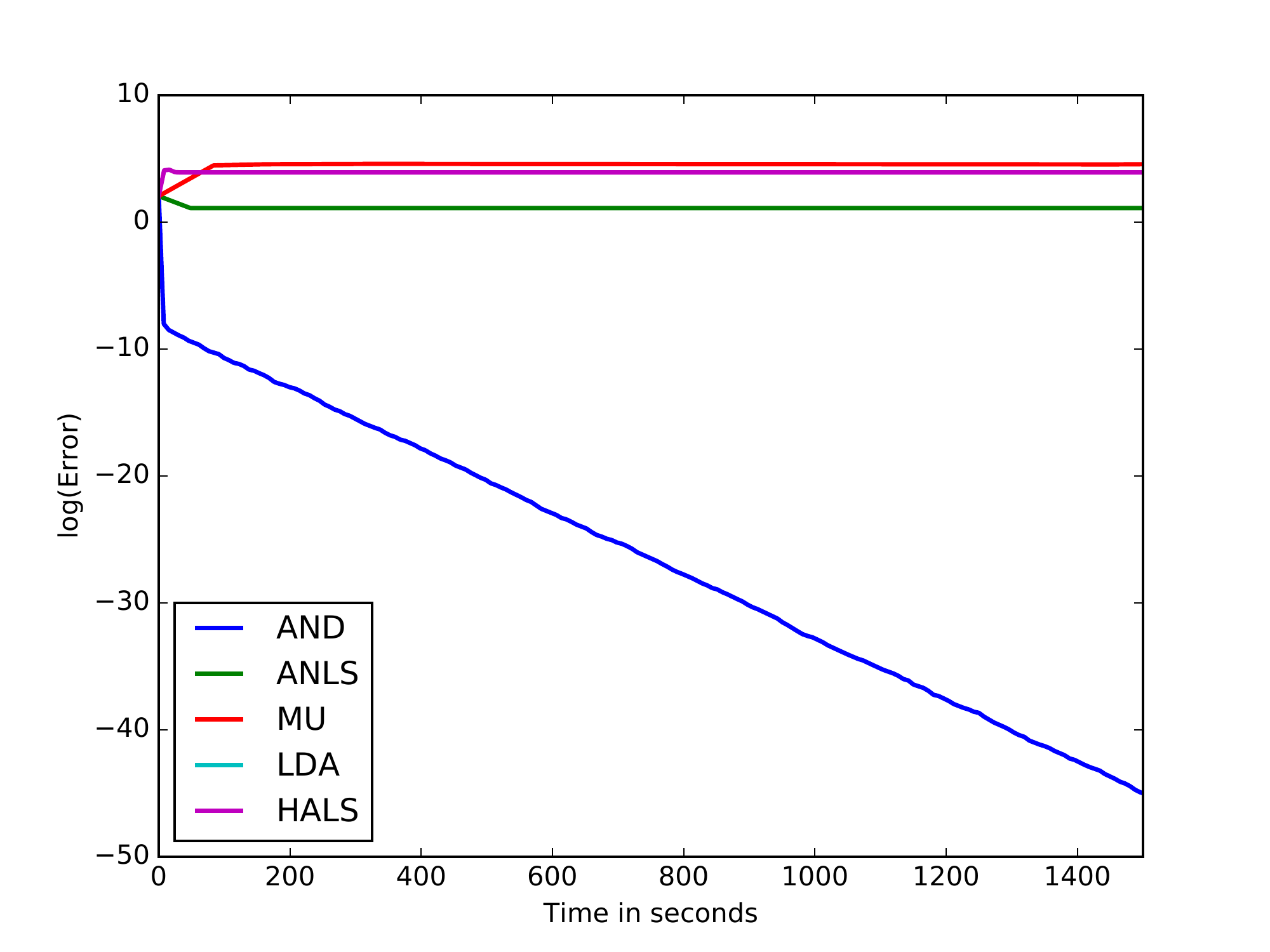}}
\caption{The performance of different algorithms on the three datasets. The $x$-axis is the running time (in seconds), the $y$-axis is the logarithm of the total correlation error. 
%The algorithm~\textsf{AND} converges in linear rate on all three datasets. The algorithm~\textsf{HALS} converges on the \textsf{DIR} and \textsf{CTM} datasets in slower rate, and doesn't converge on \textsf{NEG} where the ground-truth matrix has negative entries. The other algorithms do not converge.
} 
\label{fig:convergence}
\end{figure*}
	
\begin{figure*}[!h]
\centering
\subfigure[different thresholds on \textsf{DIR}]{\includegraphics[width=\scalef\linewidth]{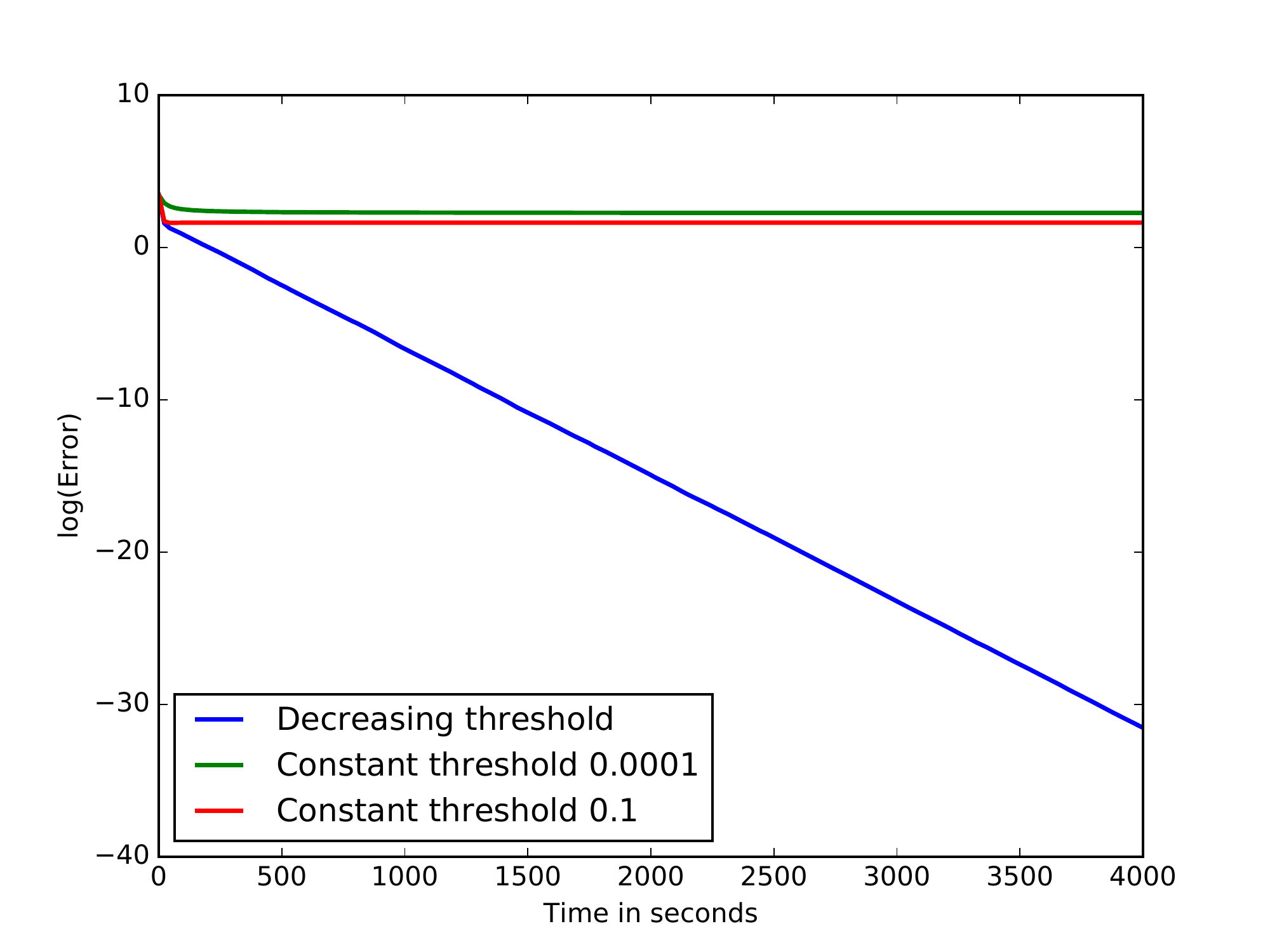}}
\subfigure[different thresholds on \textsf{CTM}]{\includegraphics[width=\scalef\linewidth]{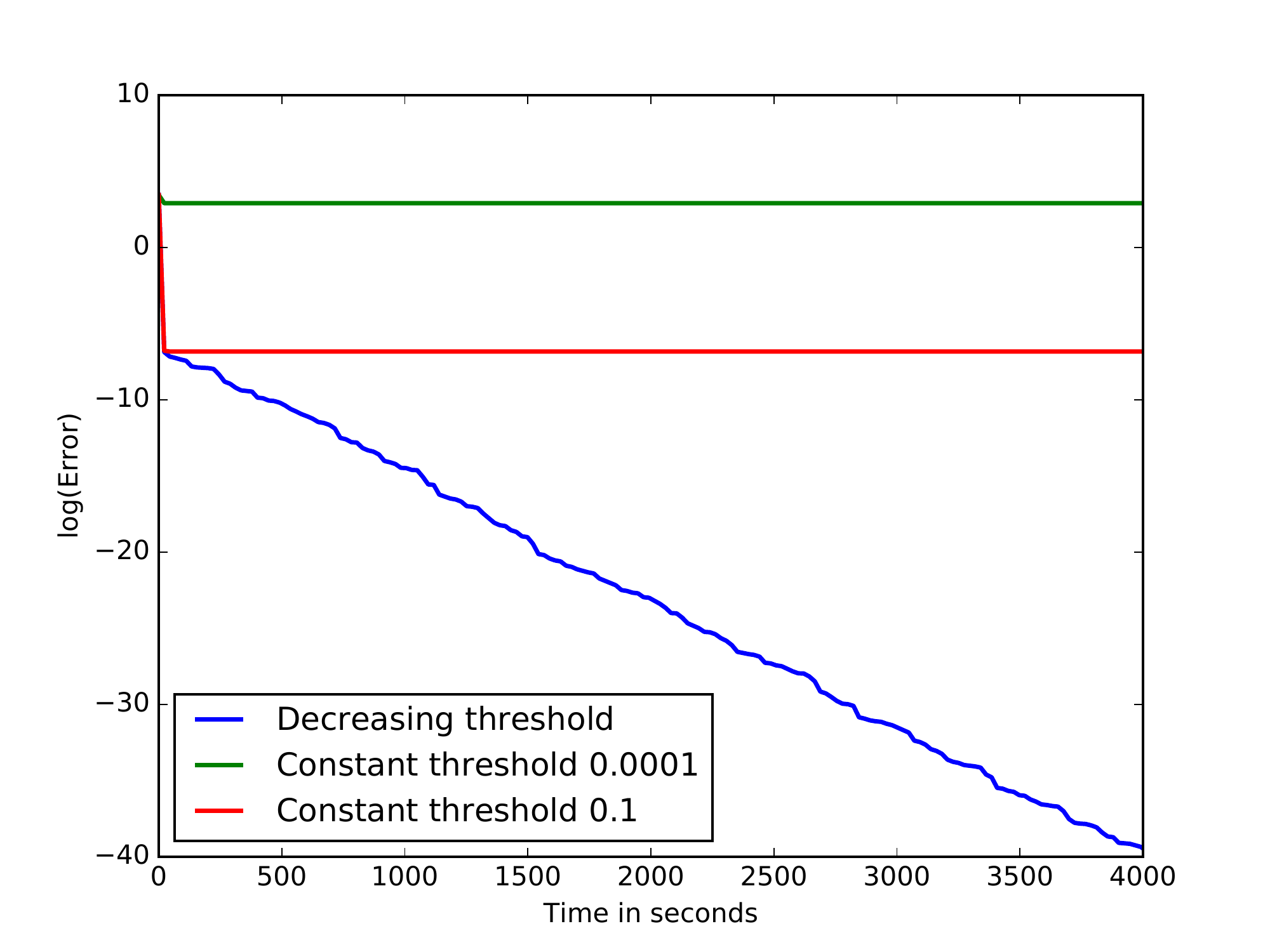}}
\subfigure[robustness to noise]{\includegraphics[width=\scalef\linewidth]{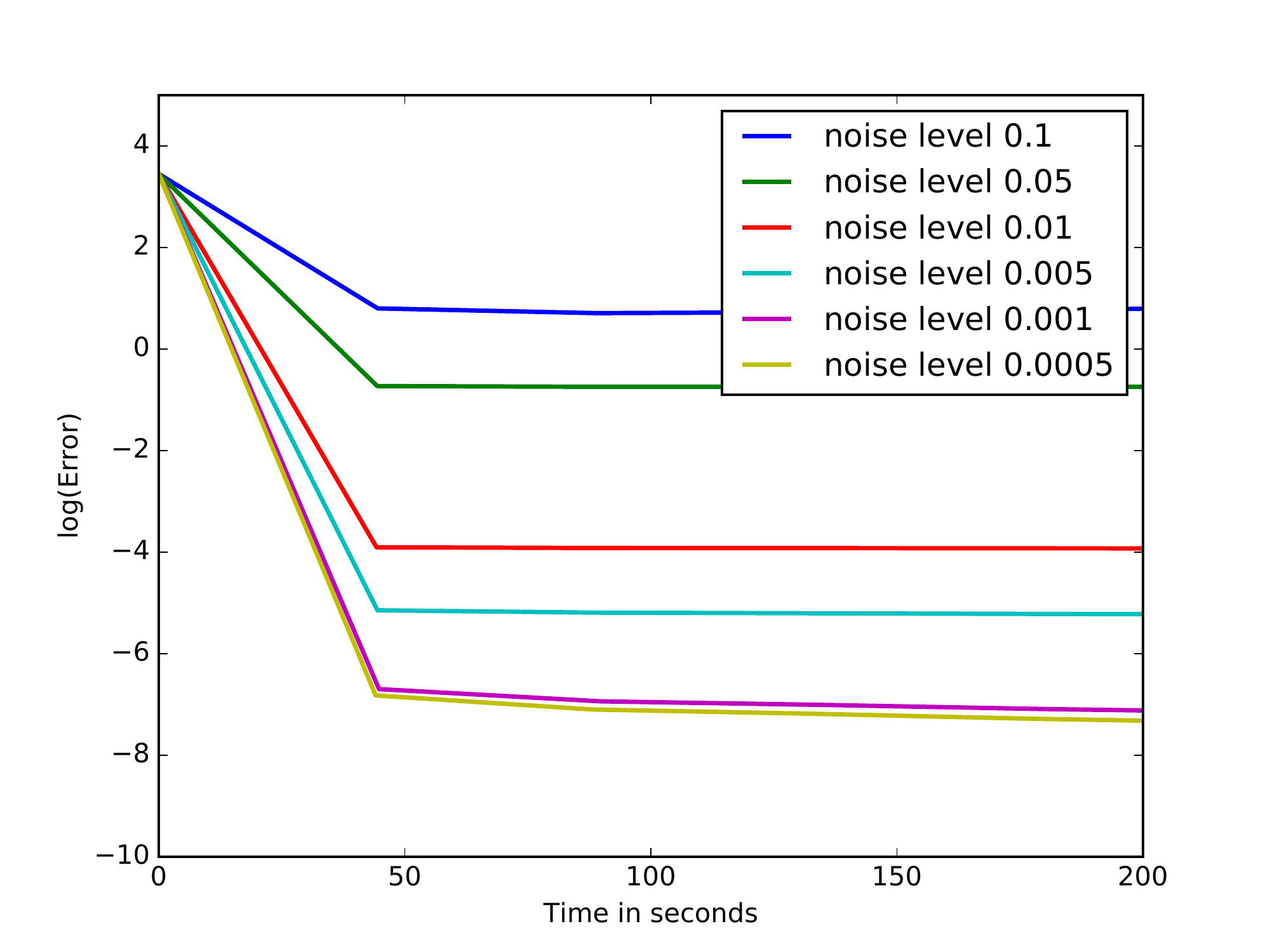}}
\caption{The performance of the algorithm \textsf{AND} with different thresholding schemes, and its robustness to noise. The $x$-axis is the running time (in seconds), the $y$-axis is the logarithm of the total correlation error. (a)(b) Using different thresholding schemes on the \textsf{DIR}/\textsf{CTM} dataset. ``Decreasing thresold'' refers to the scheme used in the original \textsf{AND}, ``Constant threshold $c$'' refers to using the threshold value $c$ throughout all iterations. 
%When using constant thresholds, the error drops at the first few steps, but then does not decrease, as predicted by our analysis.  
(c) The performance in the presence of noises of various levels. 
%The error drops at the first few steps, but then stabilizes around a constant related to the noise level, as predicted by our analysis.
}
\label{fig:threshold_noise}
\end{figure*}

\vspace{-2mm}
\paragraph{Evaluation criterion.}
Given the output matrix $\bA$ and the ground truth matrix $\bAg$, the \emph{correlation error} of the $i$-th column is given by
$$\varepsilon_i(\bA, \bAg) = \min_{j \in [D], \sigma \in \mathbb{R}} \{ \| [\bAg]^i  - \sigma [\bA]^j  \|_2 \}.$$
Thus, the error measures how well the $i$-th column of  $\bAg$ is covered by the best column of $\bA$ up to scaling. We find the best column since in some competitor algorithms, the columns of the solution $\bA$ may only correspond to a permutation of the columns of $\bAg$.\footnote{In the Algorithm~\textsf{AND}, the columns of $\bA$ correspond to the columns of $\bAg$ without permutation.} 

We also define the \emph{total correlation error} as
$$\varepsilon(\bA, \bAg) = \sum_{i = 1}^D \varepsilon_i(\bA, \bAg).$$
We report the total correlation error in all the experiments. 

\vspace{-1mm}
\paragraph{Initialization.} In all the experiments, the initialization matrix $\bA_0$ is set to 
$
\bA_0 = \bAg(\bI + \bU)
$
where $\bI$ is the identity matrix and $\bU$ is a matrix whose entries are i.i.d. samples from the uniform distribution on $[-0.05,0.05)$.
Note that this is a very weak initialization, since $[\bA_0]^i = (1+ \bU_{i,i}) [\bAg]^i + \sum_{j \neq i} \bU_{j,i} [\bAg]^j$ and the magnitude of the noise component $\sum_{j \neq i} \bU_{j,i} [\bAg]^j$ can be larger than the signal part $(1+ \bU_{i,i}) [\bAg]^i$. 
  
\vspace{-2mm}
\paragraph{Hyperparameters and Implementations.} For most experiments of \textsf{AND}, we used $T=50$ iterations for each stage, and thresholds $\alpha_i = 0.1/ (1.1)^{i-1}$. For experiments on the robustness to noise, we found $T=100$ leads to better performance. 
Furthermore, for all the experiments, instead of using one data point at each step, we used the whole dataset for update.
%For any of the other algorithms, if it has hyperparameters, we report its results for the best set of hyperparameters.

\vspace{-2mm}
\subsection{Convergence to the Ground-Truth}
Figure~\ref{fig:convergence} shows the convergence rate of the algorithms on the three datasets. 
\textsf{AND} converges in linear rate on all three datasets (note that the $y$-axis is in log-scale). \textsf{HALS} converges on the \textsf{DIR} and \textsf{CTM} datasets, but the convergence is in slower rates. Also, on \textsf{CTM}, the error oscillates. Furthermore, it doesn't converge on \textsf{NEG} where the ground-truth matrix has negative entries. \textsf{ANLS} converges on  \textsf{DIR} and \textsf{CTM}  at a very slow speed due to the non-negative least square computation in each iteration. 
\footnote{We also note that even the thresholding of \textsf{HALS} and \textsf{ALNS} designed for non-negative feature matrices is removed, they still do not converge on \textsf{NEG}.}
All the other algorithms do not converge to the ground-truth, suggesting that they do not have recovery guarantees. 

\vspace{-2mm}
\subsection{The Threshold Schemes}
Figure~\ref{fig:threshold_noise}(a) shows the results of using different thresholding schemes on \textsf{DIR}, while Figure~\ref{fig:threshold_noise}(b) shows that those on \textsf{CTM}.
When using a constant threshold for all iterations, the error only decreases for the first few steps and then stop decreasing. This aligns with our analysis and is in strong contrast to the case with decreasing thresholds. 

\vspace{-2mm}
\subsection{Robustness to Noise}
Figure~\ref{fig:threshold_noise}(c) shows the performance of \textsf{AND} on the \textsf{NOISE} dataset with various noise levels $\gamma$. The error drops at the first few steps, but then stabilizes around a constant related to the noise level, as predicted by our analysis. This shows that it can recover the ground-truth to good accuracy, even when the data have a significant amount of noise. 

%\section{Conclusion} \label{sec:conclusion}

% \section*{Acknowledgements} 

\section*{Acknowledgements} 
This work was supported in part by NSF grants CCF-1527371, DMS-1317308, Simons Investigator Award, Simons Collaboration Grant, and ONR-N00014-16-1-2329. This work was done when Yingyu Liang was visiting the Simons Institute.


\bibliography{bibfile}
\bibliographystyle{icml2017}

%\end{document} 

\newpage
\appendix
\onecolumn

\section{Complete Proofs}

We now recall the proof sketch. 

For simplicity, we only focus on one stage and the expected update.  The expected update of $\bA^{(t)}$ is given by
$$\bA^{(t + 1)} = \bA^{(t)} + \eta (\E[y z^{\top}]  - \bA^{(t)} \E[z z^{\top}]).$$

Let us write $\bA = \bAg (\bSigma_0 + \bE_0)$ where $\bSigma_0$ is diagonal and $\bE_0$ is off-diagonal. Then the decoding is given by
$$z = \phi_{\alpha} (\bA^{\dagger} x ) = \phi_{\alpha} ( (\bSigma_0 + \bE_0)^{-1 } x).$$
Let $\bSigma, \bE$ be the diagonal part and the off-diagonal part of $ (\bSigma_0 + \bE_0)^{-1 }$. 

The first step of our analysis is a key lemma for decoding. It says that under suitable conditions, $z$ will be close to $\bSigma x $ in the following sense:
$$ \E[ \bSigma xx^\top] \approx \E[z x^{\top}], ~~\E[\bSigma xz^{\top}] \approx \E[z z^{\top}].$$
This key decoding lemma is formally stated in Lemma~\ref{lem:decoding1} (for the simplified case where $x\in \{0,1\}^D$) and Lemma~\ref{lem:decoding_ii} (for the general case where $x \in [0,1]^D$).

Now, let us write $\bA^{(t)}  = \bAg (\bSigma_t + \bE_t)$. Then applying the above decoding lemma, the expected update of $\bSigma_t + \bE_t$ is 
$$\bSigma_{t + 1}+ \bE_{t + 1} = ( \bSigma_t + \bE_t )( \bI   - \bSigma \bDelta \bSigma ) + \bSigma^{-1} (\bSigma \bDelta \bSigma)  + \bR_t$$
where $\bR_t$ is a small error term.

The second step is a key lemma for updating the feature matrix: for the update rule 
$$\bSigma_{t + 1} + \bE_{t + 1}  =(  \bSigma_{t } + \bE_{t }  )(1 - \eta \bLambda) + \eta \bQ \bLambda + \eta \bR_t$$
where $\bLambda$ is a PSD matrix and $\|\bR_t \|_2\le C''$, we have
\begin{align*}
\|\bSigma_t + \bE_t - \bQ \|_2 & \le \|  \bSigma_{0} + \bE_{0 } - \bQ   \|_2 (1 -\eta \lambda_{\min}(\bLambda))^t 
   +  \frac{C''}{\lambda_{\min}(\bLambda)}. 
\end{align*}
This key updating lemma is formally stated in Lemma~\ref{lem:main_simple}. 

Applying this on our update rule with $\bQ = \bSigma^{-1}$ and $\bLambda = \bSigma \bDelta \bSigma$, we  know that when the  error term is sufficiently small, we can make progress on $\|\bSigma_t + \bE_t - \bSigma^{-1} \|_2$.
Then, by using the fact that
$\bSigma_0  \approx \bI$ and $\bE_0$ is small, and the fact that $\bSigma$ is the diagonal part of $ (\bSigma_0 + \bE_0)^{-1 }$, we can show that after sufficiently many iterations, $\|\bSigma_t - \bI\|_2$ blows up slightly, while $\|\bE_t\|_2$ is reduced significantly (See Lemma~\ref{lem:epoching} for the formal statement). 
Repeating this for multiple stages completes the proof.

\paragraph{Organization.} 
Following the proof sketch, we first present the decoding lemmas in Section~\ref{app:decoding}, and then the update lemmas in Section~\ref{app:updating}.
Section~\ref{app:main} then uses these lemmas to prove the main theorems (Theorem~\ref{thm:main_binary} and Theorem~\ref{thm:main}).
Proving the decoding lemmas is highly non-trivial, and we collect the lemmas needed in Section~\ref{app:expectation}.

Finally, the analysis for the robustness to noise follows a similar proof sketch. It is presented in Section~\ref{app:noise}.

\subsection{Decoding} \label{app:decoding}

\subsubsection{$x_i \in \{0, 1\}$}

Here we present the following decoding Lemma  when $x_i \in \{0, 1\}$.

\begin{lem}[Decoding]\label{lem:decoding1}

For every $\ell \in [0, 1)$, every off-diagonal  matrix $\bE'$ such that $\| \bE' \|_2 \le \ell$ and every diagonal matrix $\bSigma'$ such that $\|\bSigma'  -  \bI \|_2 \le \frac{1}{2} $, let $z = \phi_{\alpha }( (\bSigma'  + \bE') x  )$ for $\alpha \le \frac{1}{4}$. Then for every $\beta \in (0, 1/2]$, 
$$\| \E[ (\bSigma' x - z) x^{\top} ] \|_2, \| \E[(\bSigma' x - z) z^{\top}] \|_2 = O(C_1)$$
where $$C_1 =  \frac{kr}{D}+ \frac{m \ell^4 r^2}{\alpha^3 D^2} + \frac{\ell^2 \sqrt{km} r^{1.5}}{D^{1.5} \beta } + \frac{\ell^4 r^3 m}{\beta^2 D^2} + \frac{\ell^5 r^{2.5} m}{D^2 \alpha^2 \beta}.$$

\end{lem}

\begin{proof}[Proof of Lemma \ref{lem:decoding1}]

%Let matrix $\bM = \E[(x - \bSigma' z) z^{\top}]$. We will bound the $\ell_1$ and $\ell_{\infty}$ norm of $\bM$ and apply Matrix H\"older inequality to obtain the bound on the spectral norm. 

We will prove the bound on $\| \E[(\bSigma' x - z) z^{\top}] \|_2 $, and a simlar argument holds for that on $\| \E[ (\bSigma' x - z) x^{\top} ] \|_2$. 

First consider the term $|\bSigma_{i, i}'  x_i - z_i|$ for a fixed $i \in [D]$. 
Due to the decoding, we have $ z_i = \phi_{\alpha} ( [(\bSigma' + \bE')]_{i} x) =  \phi_{\alpha} ( \bSigma_{i, i}' x_i + \langle e_i, x \rangle )$ where $e_i$ is the $i$-th row of $\bE'$.

\begin{claim}
\begin{align} \label{eq:fajoisahaoughuafa}
\bSigma_{i, i}'  x_i - z_i   = a_{x, 1} \phi_{\alpha} (- \langle e_i, x \rangle)  + a_{x, 2} \phi_{\alpha} (\langle e_i, x \rangle) - \langle e_i, x \rangle  x_i
\end{align}
where $a_{x, 1}, a_{x, 2} \in [-1, 1]$ that depends on $x$.
\end{claim}

\begin{proof}
To see this, we split into two cases: 
\begin{enumerate}
\item When $x_i = 0$, then $|\bSigma_{i, i}'  x_i - z_i   | = |z_i | \le \phi_{\alpha} (\langle e_i, x \rangle)$. 
\item When $x_i = 1$, then $z_i  = 0$ only when $-\langle e_i, x \rangle \ge \frac{1}{2} - \alpha  \ge \alpha$, which implies that $|\bSigma_{i, i}'  x_i - z_i   + \langle e_i, x \rangle| \le  \alpha \leq  \phi_{\alpha}(-  \langle e_i, x \rangle )$. When $\bSigma_{i, i}'  x_i - z_i   + \langle e_i, x \rangle  \not= 0$, then $\bSigma_{i, i}'  x_i - z_i   = - \langle e_i, x \rangle$.
\end{enumerate}

Putting everything together, we always have:
$$|  \bSigma_{i, i}'x_i - z_i  + \langle {e}_i, x \rangle x_i |  \le  \phi_{\alpha}(| \langle e_i, x \rangle|)$$

which means that there exists $a_{x, 1}, a_{x, 2} \in [-1, 1]$ that depend on $x$ such that 
\begin{eqnarray*}   \bSigma_{i, i}'x_i - z_i   = a_{x, 1} \phi_{\alpha} (- \langle e_i, x \rangle)  + a_{x, 2} \phi_{\alpha} (\langle e_i, x \rangle) - \langle e_i, x \rangle  x_i.
\end{eqnarray*}
\end{proof}

Consider the term $ \langle e_i, x \rangle  x_i $, we know that for every $\beta \ge 0$, 
$$| \langle e_i, x \rangle   - \phi_{\beta} (\langle e_i, x \rangle ) +  \phi_{\beta} ( - \langle e_i, x \rangle ) | \le \beta.$$

Therefore, there exists $b_x \in [- \beta, \beta]$ that depends on $x$ such that
$$\langle e_i, x \rangle x_i  = \phi_{\beta} (\langle e_i, x \rangle ) -   \phi_{\beta} ( - \langle e_i, x \rangle ) - b_x x_i.$$

Putting into \eqref{eq:fajoisahaoughuafa}, we get:
$$ \bSigma_{i, i}'  x_i - z_i    = a_{x, 1} \phi_{\alpha} (- \langle e_i, x \rangle)  + a_{x, 2} \phi_{\alpha} (\langle e_i, x \rangle)  -  \phi_{\beta}(\langle {e}_i, x \rangle )x_i +   \phi_{\beta}( -\langle {e}_i, x \rangle) x_i   + b_x x_i .$$

For notation simplicity, let us now write
$$z_i = ( \bSigma_{i, i}' - b_x)  x_i  + a_i + b_i$$

where $$a_i =  -  a_{x, 1} \phi_{\alpha} (- \langle e_i, x \rangle)  - a_{x, 2} \phi_{\alpha} (\langle e_i, x \rangle), \quad b_i = \phi_{\beta}(\langle {e}_i, x \rangle )x_i -   \phi_{\beta}( -\langle {e}_i, x \rangle) x_i.$$

We then have
$$(\bSigma_{i, i}' x_i - z_i) z_j  = (b_x x_i - a_i - b_i)(( \bSigma_{j, j}' - b_x)  x_j  + a_j + b_j).$$

Let us now construct matrix $\bM_1, \cdots \bM_9$, whose entries are given by
\begin{enumerate}
\item  $(\bM_1)_{i, j} = b_x x_i ( \bSigma_{j, j}' - b_x)  x_j $
\item  $(\bM_2)_{i, j}  = b_x x_i  a_j$
\item $(\bM_3)_{i, j}  = b_x x_i b_j$
\item $(\bM_4)_{i, j} = -a_i ( \bSigma_{j, j}' - b_x)  x_j $
\item $(\bM_5)_{i, j}  = -a_i a_j$
\item $(\bM_6)_{i, j} = -a_i b_j $
\item $(\bM_7)_{i, j}  = -b_i ( \bSigma_{j, j}' - b_x)  x_j $
\item $(\bM_8)_{i, j} = - b_i a_j $
\item $(\bM_9)_{i, j}  = -b_i b_j$
\end{enumerate}

Thus, we know that $\E[(\bSigma' x -  z) z^{\top}] = \sum_{i = 1}^9\E[\bM_i]$. It is sufficient to bound  the spectral norm of each matrices separately, as we discuss below.
\begin{enumerate}
\item  $\bM_2, \bM_4$: these matrices can be bounded by Lemma \ref{lem:exp_M1}, term 1. 
\item $\bM_5$: this matrix can be bounded by Lemma \ref{lem:exp_M1}, term 2.
\item $\bM_6, \bM_8$: these matrices can be bounded by Lemma  \ref{lem:exp_M2}, term 3.
\item $\bM_3, \bM_7$: these matrices can be bounded by Lemma \ref{lem:exp_M2}, term 2. 
\item $\bM_9$: this matrix can be bounded by Lemma \ref{lem:exp_M2}, term 1. 
\item $\E[\bM_1]$: this matrix is of the form $\E[ b_x x (x \hp d_x)^{\top}]$, where $d_x$ is a vector whose $j$-th entry is $( \bSigma_{j, j}' - b_x)$. 

To bound the this term, we have that for any $u, v$ such that $\| u\|_2 = \| v\|_2 = 1$,
\begin{eqnarray*}
 u^{\top}\E[b_x x (x \hp d_x)^{\top}] v &=& \E[ b_x \langle u,  x \rangle \langle v, x \hp d_x \rangle].
 \end{eqnarray*}
When $\beta \le \frac{1}{2}$,  since $x$ is non-negative, we know that the maximum of $\E[ b_x \langle u,  x \rangle \langle v, x \hp d_x \rangle]$ is obtained when $b_x  = \beta$, $d_x = (2, \cdots , 2)$ and $u, v$ are all non-negative, which gives us
$$\E[ b_x \langle u,  x \rangle \langle v, x \hp d_x \rangle] \le 2\beta \| \E[x x^{\top} ] \|_2 \le \sqrt{\| \E[x x^{\top} ] \|_1 \| \E[x x^{\top} ] \|_{\infty}} = \| \E[x x^{\top} ] \|_1.$$

Now, for each row of $\| \E[x x^{\top} ] \|_1 $, we know that $[\E[x x^{\top} ]]_i \le \E[x_i \sum_j x_j] \le \frac{2r k}{D}$, which gives us 
$$\| \E[\bM_1]\|_2 \le \frac{4 \beta r k }{D}.$$
\end{enumerate}

Putting everything together gives the bound on $\|\E[(\bSigma' x -  z) z^{\top}] \|_2 $. A similar proof holds for the bound on $\|\E[(\bSigma' x -  z) x^{\top}] \|_2$.
%$$\|\E[(\bSigma' x -  z) x^{\top}] \bSigma' \|_2 , \|\E[(\bSigma' x -  z) z^{\top}] \|_2 = O(C_1)  $$
%Where $C_1$ is given by $$C_1 =  \frac{\beta kr}{D}+ \frac{m \ell^4 r^2}{\alpha^3 D^2} + \frac{\ell^2 \sqrt{km} r^{1.5}}{D^{1.5} \beta } + \frac{\ell^4 r^3 m}{\beta^2 D^2} + \frac{\ell^5 r^{2.5} m}{D^2 \alpha^2 \beta} $$
\end{proof}

\subsubsection{General $x_i$}

We have the following decoding lemma for the general case when $x_i \in [0,1]$ and the distribution of $x$ satisfies the order-$q$ decay condition. 

%Suppose $z = \phi_{\alpha} ( (\bSigma'  + \bE') x )$ for diagonal matrix $\bSigma'$ and off-diagonal matrix $\bE$ such that $\|\bSigma' - \bI \|_2 \le \frac{1}{2}$ and $\| \bE' \|_2 \le \ell$.
 
 \begin{lem}[Decoding II]\label{lem:decoding_ii}
 For every $\ell \in [0, 1)$, every off-diagonal  matrix $\bE'$ such that $\| \bE' \|_2 \le \ell$ and every diagonal matrix $\bSigma'$ such that $\|\bSigma'  -  \bI \|_2 \le \frac{1}{2} $, let $z = \phi_{\alpha }( (\bSigma'  + \bE') x  )$ for $\alpha \le \frac{1}{4}$. Then for every $\beta \in (0, \alpha]$, 
$$\| \E[ (\bSigma' x - z) x^{\top} ] \|_2, \| \E[(\bSigma'x - z) z^{\top}] \|_2 = O(C_2)$$
where $$C_2 =   \frac{\ell^4 r^3 m}{\alpha \beta^2 D^2} + \frac{\ell^5 r^{2.5} m}{D^2 \alpha^{2.5} \beta} + \frac{\ell^2  k r }{D \beta} \left( \frac{m}{Dk} \right)^{\frac{q}{2 q + 2}} + \frac{\ell^3 r^2 \sqrt{km}}{D^{1.5} \alpha^2}  + \frac{\ell^6 r^4 m}{\alpha^4 D^2}  + k \beta \left(\frac{r}{D} \right)^{\frac{2q + 1}{2 q + 2}}+ \frac{kr}{D} \alpha^{\frac{q + 1}{2}}.$$

 \end{lem}

 \begin{proof}[Proof of Lemma \ref{lem:decoding_ii}]
We consider the bound on $\| \E[(\bSigma'x - z) z^{\top}] \|_2$, and that on $\E[ (\bSigma' x - z) x^{\top} ] \|_2$ can be proved by a similar argument. 

Again, we still have $$ z_i = \phi_{\alpha} ( [(\bSigma' + \bE')]_{i} x) =  \phi_{\alpha} ( \bSigma_{i, i}' x_i + \langle e_i, x \rangle ).$$
 
However, this time even when $x_i \not= 0$, $x_i$ can be smaller than $\alpha$. Therefore, we need the following inequality.

\begin{claim} 
 $$|  \bSigma_{i, i}'x_i - z_i  + \langle {e}_i, x \rangle \mathsf{1}_{x_i  \geq 4 \alpha} |  \le  \phi_{\alpha/2}(| \langle e_i, x \rangle|) + 2x_i \mathsf{1}_{x_i \in (0 , 4 \alpha)}.$$
\end{claim}
 
\begin{proof}
 To see this, we can consider the following four events:
 \begin{enumerate}
 \item $x_i  = 0$, then $|  \bSigma_{i, i}'x_i - z_i  + \langle {e}_i, x \rangle \mathsf{1}_{x_i  \geq 2 \alpha} | = |z_i| \le \phi_{\alpha} (\langle e_i, x \rangle)$
 \item $x_i \ge 4\alpha$. $|  \bSigma_{i, i}'x_i - z_i  + \langle {e}_i, x \rangle \mathsf{1}_{x_i  \geq 4 \alpha} | = | \bSigma_{i, i}' x_i + \langle e_i, x \rangle  -   \phi_{\alpha} ( \bSigma_{i, i}' x_i + \langle e_i, x \rangle ) |$. Since $\bSigma_{i, i}' x_i \ge 2 \alpha$, we can get the same bound. 
 \item $x_i \in (\alpha/4, 4 \alpha)$: then if $z_i \not= 0$, $|  \bSigma_{i, i}'x_i - z_i   + \langle e_i, x \rangle| = 0$. Which implies that $$|  \bSigma_{i, i}'x_i - z_i | \leq |\langle e_i, x \rangle| \le \phi_{\alpha/2}(|\langle e_i, x \rangle|)  + \frac{\alpha}{2}  \le  \phi_{\alpha/2}(|\langle e_i, x \rangle|) + 2 x_i \mathsf{1}_{x_i \in (0 , 4 \alpha)} $$
 
 If $z_i = 0$, then 
 $$|  \bSigma_{i, i}'x_i - z_i |  =   \bSigma_{i, i}'x_i  \leq 2x_i \mathsf{1}_{x_i \in (0 , 2 \alpha)} $$
 
 \item $x_i \in (0 ,  \alpha/4)$, then $\bSigma_{i, i}'x_i  \le \frac{\alpha}{2}$, therefore, $z_i \not= 0$ only when $\langle e_i, x \rangle \ge \frac{\alpha}{2}$. We still have: $|  \bSigma_{i, i}'x_i - z_i |  \leq \phi_{\alpha/2}(\langle e_i, x \rangle)$
 
 If $z_0 = 0$, then $|  \bSigma_{i, i}'x_i - z_i |  \le 2 x_i \mathsf{1}_{x_i \in (0 , 4 \alpha)} $ as before. 
 \end{enumerate} 
 
 Putting everything together, we have the claim.
 \end{proof}
  
Following the exact same calculation as in Lemma \ref{lem:decoding1}, we can obtain 
\begin{align*} 
\bSigma_{i, i}'  x_i - z_i  & = a_{x, 1} \phi_{\alpha/2} (- \langle e_i, x \rangle)  + a_{x, 2} \phi_{\alpha/2} (\langle e_i, x \rangle)  \\
& -  \phi_{\beta}(\langle {e}_i, x \rangle )  \mathsf{1}_{x_i \geq 4 \alpha} +   \phi_{\beta}( -\langle {e}_i, x \rangle)  \mathsf{1}_{x_i \geq 4 \alpha}    \\
& + b_x  \mathsf{1}_{x_i \geq 4 \alpha}  + c_x 2x_i \mathsf{1}_{x_i \in (0 , 4 \alpha)}
\end{align*}
 for $a_{x, 1}, a_{x, 1} , c_{x} \in [-1, 1]$ and $b_{x} \in [- \beta, \beta]$. 

 Therefore, consider a matrix $\bM$ whose $(i, j)$-th entry is $ (\bSigma_{i, i}'  x_i  - z_i )z_j$. This entry can be written as the summation of the following terms.
 \begin{enumerate}
 \item Terms that can be bounded by Lemma \ref{lem:exp_M1}. These include 
 $$ a_{x, 1} \phi_{\alpha/2} (- \langle e_i, x \rangle)  x_j,  \quad a_{x, 2} \phi_{\alpha/2} (\langle e_i, x \rangle)  x_j, \quad a_{x, u}a_{x, v} \phi_{\alpha/2} ((-1)^u \langle e_i, x \rangle)   \phi_{\alpha/2} ((-1)^v \langle e_j, x \rangle) $$
 for $u, v \in \{ 1, 2\}$, 
 and $$a_{x, u}b_x  \phi_{\alpha/2} ((-1)^u \langle e_i, x \rangle)  \mathsf{1}_{x_i \geq 4 \alpha}, \quad  2a_{x, u}c_x  \phi_{\alpha/2} ((-1)^u\langle e_i, x \rangle)  x_j \mathsf{1}_{x_j \in (0, 4 \alpha)} $$
 by using $0 \le \mathsf{1}_{x_j \ge 4 \alpha}  \le \frac{x_j}{4 \alpha}$ and $0 \le x_j \mathsf{1}_{x_j \in (0 , 4 \alpha)}  \le x_j$. 
 \item Terms that can be bounded by Lemma \ref{lem:exp_M3}. These include 
 $$- \phi_{\beta}(\langle {e}_i, x \rangle )  \mathsf{1}_{x_i \geq 4 \alpha} x_j,  \quad \phi_{\beta}( -\langle {e}_i, x \rangle)  \mathsf{1}_{x_i \geq 4 \alpha}    x_j ,$$
$$(-1)^u a_{x, v}\phi_{\beta}((-1)^{1 + u}\langle {e}_i, x \rangle )  \mathsf{1}_{x_i \geq 4 \alpha}    \phi_{\alpha/2} ((-1)^v \langle e_j, x \rangle),$$
 $$  (-1)^{u+ v}\phi_{\beta}((-1)^{1 + u}\langle {e}_i, x \rangle )  \mathsf{1}_{x_i \geq 4 \alpha}  \mathsf{1}_{x_j \geq 4 \alpha} \phi_{\beta}((-1)^{1 + v}\langle {e}_j, x \rangle )  $$
for $u, v \in \{ 1, 2\}$.
Also include 
 $$  2(-1)^u c_x \phi_{\beta}((-1)^{1 + u}\langle {e}_i, x \rangle )   \mathsf{1}_{x_i \geq 4 \alpha}x_j \mathsf{1}_{x_j \in (0 , 4 \alpha)}   $$
 
 by using$0 \le x_j \mathsf{1}_{x_j \in (0 , 4 \alpha)}  \le x_j$. 
Also include
  $$  2(-1)^u b_x \phi_{\beta}((-1)^{1 + u}\langle {e}_i, x \rangle )   \mathsf{1}_{x_i \geq 4 \alpha} \mathsf{1}_{x_j \ge 4 \alpha}   $$
  by using  $0 \le \mathsf{1}_{x_j \ge 4 \alpha}  \le \frac{x_j}{4 \alpha}$.
 
\item Terms that can be bounded by Lemma \ref{lem:tch}. These include
 $$b_x  \mathsf{1}_{x_i \geq 4 \alpha} x_j, \quad b_x^2  \mathsf{1}_{x_i \geq 4 \alpha}   \mathsf{1}_{x_j \geq 4 \alpha}    , \quad 2b_xc_x  \mathsf{1}_{x_i \geq 4 \alpha}x_i \mathsf{1}_{x_j \in (0 , 4 \alpha)} x_j.$$
 
 Where agin we use the fact that $0\le \mathsf{1}_{x_i \geq 4 \alpha}   \le \frac{x_j}{4 \alpha}$ and $0 \le x_i,  \mathsf{1}_{x_j \in (0 , 4 \alpha)}  \le 1$
 
\item Terms that can be bounded by Lemma \ref{lem:tc}. These include
 $$c_x 2x_i \mathsf{1}_{x_i \in (0 , 4 \alpha)} x_j, \quad 4 c_x^2 x_i \mathsf{1}_{x_i \in (0 , 4 \alpha)}  x_j \mathsf{1}_{x_j \in (0 , 4 \alpha)}.$$
 
 Where we use the fact that $0 \le \mathsf{1}_{x_j \in (0 , 4 \alpha)} \le 1$. 
 \end{enumerate}

\yingyu{need a bit more on how to bound the last two kinds of terms.}

Putting everything together, when $0 < \beta \le \alpha$, 
 $$\| \E[(\bSigma' x - z)z^{\top} ]\|_2 = O\left( C_2 \right) $$

 where $$C_2 = \frac{\ell^4 r^3 m}{\alpha \beta^2 D^2} + \frac{\ell^5 r^{2.5} m}{D^2 \alpha^{2.5} \beta} + \frac{\ell^2  k r }{D \beta} \left( \frac{m}{Dk} \right)^{\frac{q}{2 q + 2}} + \frac{\ell^3 r^2 \sqrt{km}}{D^{1.5} \alpha^2}  + \frac{\ell^6 r^4 m}{\alpha^4 D^2}  + k \beta \left(\frac{r}{D} \right)^{\frac{2q + 1}{2 q + 2}}+ \frac{kr}{D} \alpha^{\frac{q + 1}{2}}.$$

This gives the bound on $\E[ (\bSigma' x - z) z^{\top} ] \|_2$.
The bound on $\E[ (\bSigma' x - z) x^{\top} ] \|_2$ can be proved by a similar argument.

 \end{proof}

\subsection{Update} \label{app:updating}
\subsubsection{General Update Lemma}

\begin{lem}[Update]\label{lem:main_simple}
Suppose $\bSigma_{t }$ is diagonal and $\bE_{t }$ is off-diagonal for all $t$.
Suppose we have an update rule that is given by
$$\bSigma_{t + 1} + \bE_{t + 1}  =(  \bSigma_{t } + \bE_{t }  )(1 - \eta \bDelta) + \eta \bSigma \bDelta + \eta \bR_t$$
for some positive semidefinite matrix $\bDelta$ and some $\bR_t$ such that $\|\bR_t \|_2\le C''$. Then for every $t \ge 0$, 
%$$\|\bSigma_t + \bE_t - \bSigma \|_2 \le \|  \bSigma_{t } + \bE_{t }   \|_2 (1 -\eta \lambda_{\min}(\bDelta))^t + \| \bSigma\|_2 (1 -\eta \lambda_{\min}(\bDelta))^{t }  + \eta C'' t$$

$$\|\bSigma_t + \bE_t - \bSigma \|_2 \le \|  \bSigma_{0} + \bE_{0 } - \bSigma   \|_2 (1 -\eta \lambda_{\min}(\bDelta))^t  +  \frac{C''}{\lambda_{\min}(\bDelta)}. $$
\end{lem}

\begin{proof}[Proof of Lemma \ref{lem:main_simple}]
We know that the update is given by
$$\bSigma_{t + 1} + \bE_{t + 1} - \bSigma = (  \bSigma_{t } + \bE_{t } - \bSigma  )(1 - \eta \bDelta)  + \eta \bR_t.$$

If we let 
$$\bSigma_{t } + \bE_{t } - \bSigma =  (\bSigma_{0} + \bE_{0} - \bSigma)(1 - \eta \bDelta)^t + \bC_t.$$

Then we can see that the update rule of $\bC_t$ is given by
\begin{align*}
\bC_0 & = 0, \\
\bC_{t + 1} & = \bC_t (1 - \eta \bDelta) + \eta \bR_t
\end{align*}

which implies that $\forall t \ge 0, \| \bC_t \|_2 \le \frac{C''}{\lambda_{\min}(\bDelta)}$. 

Putting everything together completes the proof. 
\end{proof}

\begin{lem}[Stage]\label{lem:epoching}
In the same setting as Lemma \ref{lem:main_simple}, suppose initially for $\ell_1, \ell_2 \in [0, \frac{1}{8})$ we have 
$$\|\bSigma_0  - \bI \|_2 \le \ell_1, \|\bE_0\| \le \ell_2, \bSigma = (\diag[(\bSigma_0 + \bE_0)^{-1}])^{-1}.$$ 
Moreover, suppose in each iteration, the error $\bR_t$ satisfies that $\|\bR_t \|_2 \le \frac{\lambda_{\min}(\bDelta)}{160} \ell_2$.

Then after $t = \frac{\log \frac{400}{\ell_2}}{\eta \lambda_{\min} (\bDelta)} $ iterations, we have
\begin{enumerate}
\item $\|\bSigma_t - \bI \|_2 \le  \ell_1 + 4 \ell_2$,
\item $\|\bE_t \|_2 \le \frac{1}{40} \ell_2 $.
\end{enumerate}
\end{lem}

\begin{proof}[Proof of Lemma \ref{lem:epoching}]
Using Taylor expansion, we know that 
$$\diag[(\bSigma_0 + \bE_0)^{-1}] =  \bSigma_0^{-1} + \sum_{i = 1}^{\infty} \bSigma_0^{-1/2} \diag[( - \bSigma_0^{-1/2} \bE_0 \bSigma_0^{-1/2})^i] \bSigma_0^{-1/2}.$$

Since $\|\diag(\bM)\|_2 \le \|\bM\|_2$ for any matrix $\bM$, 
\begin{align*}
\|\diag[(\bSigma_0 + \bE_0)^{-1}]  - \bSigma_0^{-1} \|_2  & =  \|\sum_{i = 1}^{\infty} \bSigma_0^{-1/2} \diag[( - \bSigma_0^{-1/2} \bE_0 \bSigma_0^{-1/2})^i] \bSigma_0^{-1/2} \|_2 \\
& \le \|\sum_{i = 1}^{\infty} \bSigma_0^{-1/2} ( - \bSigma_0^{-1/2} \bE_0 \bSigma_0^{-1/2})^i\bSigma_0^{-1/2} \|_2 \\
& = \| [(\bSigma_0 + \bE_0)^{-1}] ( -  \bE_0 \bSigma_0^{-1}) \|_2\\
& \le \frac{\ell_2}{(1 - \ell_1)(1 - \ell_1 - \ell_2)}\le \frac{32}{21} \ell_2.
\end{align*}

Therefore, 
$$\|\diag[(\bSigma_0 + \bE_0)^{-1}]  \bSigma_0 - \bI \|_2 \le \frac{\ell_2 (1 + \ell_1) }{(1 - \ell_1)(1 - \ell_1 - \ell_2)} \le \ell := \frac{12}{7} \ell_2.$$

which gives us 
$$\|\diag[(\bSigma_0 + \bE_0)^{-1}]^{-1}  \bSigma_0^{-1} - \bI \|_2  \le \frac{\ell}{1 - \ell} \le \frac{24}{11}\ell_2.$$

This then leads to 
$$\| \bSigma - \bSigma_0 \|_2 \le \|\diag[(\bSigma_0 + \bE_0)^{-1}]^{-1} - \bSigma_0 \|_2  \le \frac{(1 + \ell_1)\ell}{1 - \ell} \le 3 \ell_2.$$

Now since $\|\bSigma_0 + \bE_0 - \bSigma \|_2 \le   4\ell_2  \le 1$,  after $t  =  \frac{\log \frac{400}{\ell_2}}{\eta \lambda_{\min} (\bDelta)} $ iterations, we have

$$\|\bSigma_t + \bE_t - \bSigma \|_2  \le \frac{1}{80} \ell_2.$$

Then since $\bSigma_t  - \bSigma = \diag[\bSigma_t + \bE_t - \bSigma]$, we have
$$\|\bSigma_t - \bSigma \|_2 \le \|\bSigma_t + \bE_t - \bSigma \|_2  \le \frac{1}{80} \ell_2.$$

This implies that 
$$\|\bE_t \|_2 \le \|\bSigma_t - \bSigma \|_2  + \|\bSigma_t + \bE_t - \bSigma \|_2 \le \frac{1}{40} \ell_2$$
and
$$\| \bSigma_t - \bI \|_2 \le  \|\bSigma_t - \bSigma \|_2  + \| \bSigma - \bSigma_0 \|_2 + \| \bSigma_0 - \bI \|_2  \le \frac{1}{80}\ell_2  + 3 \ell_2 + \ell_1 \le  \ell_1 + 4 \ell_2.$$
%
%$$\| \bSigma_t^{-1} - \bI \|_2 \le \| \bSigma^{-1} - \bI \|_2 +  \frac{1}{20}\ell_2 \le  \ell_1 + 10 \ell_2 $$

%$$\|\bSigma_t + \bE_t - \bSigma \|_2  \le \frac{1}{40} \ell_2$$
%Using Taylor expansion, we know that 
%$$\bSigma = \bSigma_0^{-1} + \sum_{i = 1}^{\infty} \bSigma_0^{-1/2} ( - \bSigma_0^{-1/2} \bE_0 \bSigma_0^{-1/2})^i \bSigma_0^{-1/2} $$
%
%Which implies that $$\| \bSigma - \bI \|_2 \le \|  \bSigma_0^{-1} - \bI \|_2 + \frac{\ell_2}{(1 - \ell_1)(1 - \ell_1 - \ell_2)}  \le \ell_1 + 2 \ell_2$$ 
%
%Using Taylor expansion again, we know that 
%$$\| \bSigma^{-1} - \bI \|_2 \le \|  \bSigma_0 - \bI \|_2 +  \frac{2 \ell_2}{(1 - \ell_1 - 2 \ell_2)(1 - \ell_1 - 4\ell_2)} \le \ell_1 + 9 \ell_2 $$
%
%
%
%
%We know that $\|\bSigma_0 + \bE_0 - \bSigma \|_2 \le  2\ell_1 + 3\ell_2  \le 2$, therefore, after $t  =  \frac{\log \frac{200}{\ell_2}}{\eta \lambda_{\min} (\bDelta)} $ iterations, we have:
%$$\|\bSigma_t + \bE_t - \bSigma \|_2  \le \frac{1}{40} \ell_2$$
%
%
%Which implies that $\|\bE_t \|_2 \le \frac{1}{40} \ell_2$ since $\bSigma_t  - \bSigma$ is diagonal. Then 
%$$\| \bSigma_t - \bI \|_2 \le \| \bSigma - \bI \|_2 +  \frac{1}{20}\ell_2 \le  \ell_1 + 3 \ell_2$$
%
%$$\| \bSigma_t^{-1} - \bI \|_2 \le \| \bSigma^{-1} - \bI \|_2 +  \frac{1}{20}\ell_2 \le  \ell_1 + 10 \ell_2 $$
\end{proof}

\begin{cor}[Corollary of Lemma \ref{lem:epoching}] \label{cor:epoching}
Under the same setting as Lemma \ref{lem:epoching}, suppose initially $\ell_1 \le \frac{1}{17}$, then 
\begin{enumerate}
\item $\ell_1 \le \frac{1}{8}$ holds true through all stages, \label{item:case1}
\item $\ell_2 \le \left(\frac{1}{40} \right)^{t}$ after $t$ stages.\label{item:aghoiahuoijfasfjoasis}
\end{enumerate}

\end{cor}

\begin{proof}[Proof of Corollary \ref{cor:epoching}]
The second claim is trivial. For the first claim, we have
$$(\ell_1)_{stage \ s + 1} \le (\ell_1)_{stage \ s } + 4 (\ell_2)_{stage \ s } \le \cdots \le \frac{1}{17} + \frac{1}{8}\sum_{i}(1/40)^i \le \frac{1}{8}.$$
\end{proof}

\subsection{Proof of the Main Theorems} \label{app:main}

With the update lemmas, we are ready to prove the main theorems.

\begin{proof}[Proof of Theorem~\ref{thm:main_binary}]

For simplicity, we only focus on the expected update. The on-line version can be proved directly from this by noting that the variance of the update is polynomial bounded and setting accordingly a polynomially small $\eta$. The expected update of $\bA^{(t)}$ is given by
$$\bA^{(t + 1)} = \bA^{(t)} + \eta (\E[y z^{\top}]  - \bA^{(t)} \E[z z^{\top}])$$

Let us pick $\alpha = \frac{1}{4}$, focus on one stage and write $\bA = \bAg (\bSigma_0 + \bE_0)$. Then the decoding is given by
$$z = \phi_{\alpha} (\bA^{\dagger} x ) = \phi_{\alpha} ( (\bSigma_0 + \bE_0)^{-1 } x).$$

Let $\bSigma, \bE$ be the diagonal part and the off diagonal part of $ (\bSigma_0 + \bE_0)^{-1 }$. By Lemma \ref{lem:decoding1},
$$\|\E[(\bSigma x - z) x^{\top} \bSigma] \|_2, \| \E[(\bSigma x - z) z^{\top} \|_2  = O(C_1).$$

Now, if we write $\bA^{(t)}  = \bAg (\bSigma_t + \bE_t)$, then the expected update of $\bSigma_t + \bE_t$ is given by
$$\bSigma_{t + 1}+ \bE_{t + 1} = ( \bSigma_t + \bE_t )( \bI   - \bSigma \bDelta \bSigma ) + \bSigma^{-1} (\bSigma \bDelta \bSigma)  + \bR_t$$

where $\|\bR_t\|_2 = O(C_1) $. 

By Lemma \ref{lem:epoching}, as long as $C_1 = O (\sigma_{\min} (\bDelta) \| \bE_0\|_2) = O\left(\frac{k \lambda}{D} \| \bE_0 \|_2 \right)$, we can make progress. Putting in the expression of $C_1$ with $\ell \ge \| \bE_0 \|_2$, we can see that as long as
$$  \frac{\beta kr}{D}+ \frac{m \ell^4 r^2}{\alpha^3 D^2} + \frac{\ell^2 \sqrt{km} r^{1.5}}{D^{1.5} \beta } + \frac{\ell^4 r^3 m}{\beta^2 D^2} + \frac{\ell^5 r^{2.5} m}{D^2 \alpha^2 \beta}  = O\left(\frac{k \lambda}{D}\ell \right),$$

we can make progress. By setting $\beta = O\left( \frac{\lambda \ell}{r} \right)$, with Corollary \ref{cor:epoching} we completes the proof. 
\end{proof}

\begin{proof}[Proof of Theorem~\ref{thm:main}]

For simplicity, we only focus on the expected update. The on-line version can be proved directly from this by setting a polynomially small $\eta$. The expected update of $\bA^{(t)}$ is given by
$$\bA^{(t + 1)} = \bA^{(t)} + \eta (\E[y z^{\top}]  - \bA^{(t)} \E[z z^{\top}]).$$

Let us focus on one stage and write $\bA = \bAg (\bSigma_0 + \bE_0)$. Then the decoding is given by
$$z = \phi_{\alpha} (\bA^{\dagger} x ) = \phi_{\alpha} ( (\bSigma_0 + \bE_0)^{-1 } x).$$

Let $\bSigma, \bE$ be the diagonal part and the off diagonal part of $ (\bSigma_0 + \bE_0)^{-1 }$. By Lemma \ref{lem:decoding_ii},
$$\|\E[(\bSigma x - z) x^{\top} \bSigma] \|_2, \| \E[(\bSigma x - z) z^{\top} \|_2  = O(C_2).$$

Now, if we write $\bA^{(t)}  = \bAg (\bSigma_t + \bE_t)$, then the expected update of $\bSigma_t + \bE_t$ is given by
$$\bSigma_{t + 1}+ \bE_{t + 1} = ( \bSigma_t + \bE_t )( \bI   - \bSigma \bDelta \bSigma ) + \bSigma^{-1} (\bSigma \bDelta \bSigma)  + \bR_t$$

where $\|\bR_t\|_2 = O(C_2) $. 

By Lemma \ref{lem:epoching}, as long as $C_2 = O (\sigma_{\min} (\bDelta) \| \bE_0\|_2) = O\left(\frac{k \lambda}{D} \| \bE_0 \|_2 \right)$, we can make progress. Putting in the expression of $C_2$ with $\ell \ge \| \bE_0 \|_2$, we can see that as long as
$$ C_2 = \frac{\ell^4 r^3 m}{\alpha \beta^2 D^2} + \frac{\ell^5 r^{2.5} m}{D^2 \alpha^{2.5} \beta} + \frac{\ell^2  k r }{D \beta} \left( \frac{m}{Dk} \right)^{\frac{q}{2 q + 2}} + \frac{\ell^3 r^2 \sqrt{km}}{D^{1.5} \alpha^2}  + \frac{\ell^6 r^4 m}{\alpha^4 D^2}  + k \beta \left(\frac{r}{D} \right)^{\frac{2q + 1}{2 q + 2}}+ \frac{kr}{D} \alpha^{\frac{q + 1}{2}} = O\left(\frac{k \lambda}{D}\ell \right),$$

we can make progress. 
Now set
$$\beta = \frac{\lambda \ell}{D \left( \frac{r}{D} \right)^{\frac{2q + 1}{2 q + 2}}}, \alpha = \left( \frac{\lambda \ell }{r} \right)^{\frac{2}{q + 1}}$$

and thus in $C_2$,
 \begin{enumerate}
 \item First term $$\frac{\ell^{2 - \frac{2}{q + 1}} k^{0} r^{\frac{5 q + 6}{q + 1}} m^{}}{\lambda^{2 + \frac{2}{q + 1}} D^{2 - \frac{1}{ q + 1}}}$$
  \item Second term $$\frac{\ell^{4  - \frac{5}{q + 1}} k^{0} r^{\frac{7 q + 16}{2 q + 2}}m^{}}{\lambda^{1 + \frac{5}{q + 1}} D^{2 - \frac{1}{2 q + 2}}}$$
   \item Third term $$\frac{\ell^{1} k^{\frac{q + 2}{2 q + 2}} r^{\frac{4 q + 3}{2 q + 2}}m^{\frac{q}{2 q + 2}}}{\lambda^{} D^{\frac{3q + 1}{2 q + 2}}}$$
    \item Fourth term $$\frac{\ell^{3 - \frac{4}{q + 1}} k^{\frac{1}{2}} r^{\frac{4}{ q+ 1} + 2} m^{\frac{1}{2}}}{\lambda^{\frac{4}{q  +1}} D^{- \frac{3}{2}}}$$
     \item Fifth term $$\frac{\ell^{6 - \frac{8}{q + 1}} k^{0} r^{4  + \frac{8}{q + 1}}m}{\lambda^{\frac{8}{q + 1}} D^{2}}$$
 \end{enumerate}
 We need each term to be smaller than $\frac{\lambda k \ell}{D}$, which implies that  (we can ignore the constant $\ell$ )
  \begin{enumerate}
 \item First term: $$m \le \frac{k^{} D^{1 -  \frac{1}{q + 1}} \lambda^{3 + \frac{2}{q + 1}}}{  r^{5  + \frac{1}{q + 1}}}$$
  \item Second term: $$m \le \frac{ k^{ }  D^{1 - \frac{1}{2 q + 2}} \lambda^{2 + \frac{5}{q + 1}}}{r^{\frac{7}{2} + \frac{9}{2q + 2}}}$$
   \item Third term: $$m \le \frac{k D^{\frac{q - 1}{q}}\lambda^{4 + \frac{4}{q}}}{ r^{4 + \frac{2}{q}}}$$
    \item Fourth term: $$m \le \frac{k D^{} \lambda^{2 + \frac{8}{q + 1}}}{ r^{4 + \frac{8}{q + 1}}}$$
     \item Fifth term: $$m \le \frac{k D^{} \lambda^{1 + \frac{8}{q + 1}}}{ r^{4 + \frac{8}{q + 1}}}$$
 \end{enumerate}
This is satisfied by our choice of $m$ in the theorem statement.

Then with Corollary \ref{cor:epoching} we completes the proof. 
\end{proof}

\subsection{Expectation Lemmas} \label{app:expectation}

In this subsection, we assume that $x$ follows $(r, k, m, \lambda)$-GCC. Then we show the following lemmas.

\subsubsection{Lemmas with only GCC}

\begin{lem}[Expectation]\label{lem:exp1}
For every $\ell \in [0, 1)$, every vector $e$ such that $\| e \|_2 \le \ell$, for every $\alpha$ such that $\alpha > 2 \ell$, we have
$$\E[ \phi_{\alpha } (\langle e, x \rangle ) ] \le \frac{16 m \ell^4 r^2}{\alpha^2(\alpha - 2 \ell) D^2}.$$
\end{lem}

\begin{proof}[Proof of Lemma \ref{lem:exp1}]

Without lose of generality, we can assume that all the entries of $e$ are non-negative. Let us denote a new vector $g$ such that 
$$g_i =  \left\{ \begin{array}{ll}
         e_i & \mbox{if $e_i \geq \frac{\alpha}{2r}$};\\
        0 & \mbox{otherwise}.\end{array} \right.$$

Due to the fact that $\| x\|_1 \le r$, we can conclude $\langle e - g, x \rangle \le \frac{\alpha}{2 r} \times r = \frac{\alpha}{2}$, which implies
$$\phi_{\frac{\alpha}{2}} (\langle g, x\rangle ) \ge \frac{1}{2} \phi_{\alpha} (\langle e, x \rangle).$$

Now we can only focus on $g$. Since $\| g\|_2 \le \ell$,  we know that $g$ has at most $\frac{4 \ell^2 r^2}{\alpha^2}$ non-zero entries.  Let us then denote the set of non-zero entries of $g$ as $\set{E}$, so we have $| \set{E} | \le \frac{4 \ell^2 r^2}{\alpha^2} $.

Suppose the all the $x$ such that $\langle g, x \rangle \ge \frac{\alpha}{2}$ forms a set $\set{S}$ of size $S$, each $x^{(s)} \in \set{S}$ has probability $p_t$. Then we have:
$$\E[ \phi_{\alpha } (\langle e, x \rangle ) ] \le 2 \sum_{s \in [S]} p_s \langle g, x^{(s)}  \rangle = 2 \sum_{s \in [S], i \in \set{E}} p_s g_i x^{(s)}_i.$$

On the other hand, we have: 
\begin{enumerate}
\item $\forall s \in [S]: \quad  \sum_{i \in \set{E}} g_i x^{(s)}_i \ge \frac{\alpha}{2} $.  \label{item:vhaofnqiejglkjafajkfhasfjka}
\item $\forall i \not= j \in [D]: \  \sum_{s \in [S]} p_s x^{(s)}_i x^{(s)}_j \le \frac{m}{D^2}$. This is by assumption 5 of the distribution of $x$. 
\label{item:ahfoadghiafjhqiofqoipfjaofi}
\end{enumerate}

Using $(\ref{item:ahfoadghiafjhqiofqoipfjaofi})$ and multiply both side by $g_i g_j$, we get
$$\sum_{s \in [S]} p_s (g_i x^{(s)}_i ) (g_j x^{(s)}_j) \le \frac{m g_i g_j}{D^2 }$$

Sum over all $j \in \set{E}, j \not= i$, we have:
$$\sum_{s \in [S]} \sum_{j \in \set{E}, j \not= i } p_s (g_i x^{(s)}_i ) (g_j x^{(s)}_j) \le \frac{m g_i}{D^2 } \left(\sum_{j \in \set{E}, j \not=i }g_j \right) \le \frac{m g_i}{D^2 } \sum_{j \in \set{E} }g_j  $$

By $(\ref{item:vhaofnqiejglkjafajkfhasfjka})$, and since $\sum_{j \in \set{E}} g_j x^{(s)}_j \ge \frac{\alpha}{2} $ and $g_i \le \ell, x^{(s)}_i \le 1$, we can obtain $\sum_{j \in \set{E}, j \not= i} g_j x^{(s)}_j \ge \frac{\alpha}{2} - \ell$. This implies 
$$\sum_{s \in [S]} p_s (g_i x^{(s)}_i ) \le \frac{1}{ \sum_{j \in \set{E}, j \not= i } g_j x^{(s)}_j } \left( \frac{m g_i}{D^2 } \sum_{j \in \set{E} }g_j   \right) \le \frac{2 m}{(\alpha - 2 \ell )D^2} g_i \sum_{j \in \set{E} } g_j.$$

Summing over $i$,
$$\sum_{s \in [S], i \in \set{E}} p_s g_i x^{(s)}_i  \le \frac{2 m}{(\alpha - 2 \ell )D^2} \left(\sum_{j \in \set{E} } g_j \right)^2 \le  \frac{2 m}{(\alpha - 2 \ell )D^2} | \set{E} | \| g\|_2^2 \le \frac{8 m \ell^4   r^2}{\alpha^2(\alpha - 2 \ell) D^2}.$$

Putting everything together we complete the proof.
\end{proof}

\begin{lem}[Expectation, Matrix] \label{lem:exp_M1}

For every $\ell, \ell' \in [0, 1)$, every matrices $\bE, \bE' \in \mathbb{R}^{D \times D}$ such that $\| \bE \|_2, \| \bE' \|_2 \le \ell$, $\alpha \ge 4 \ell$ and every $b_x \in [-1, 1]$ that depends on $x$, the following hold.
\begin{enumerate}
\item Let $\bM$ be a matrix such that  $[\bM]_{i, j} = b_x \phi_{\alpha}(\langle [\bE]_i, x \rangle)  x_j$, then
$$\|\E[\bM] \|_2 \le \sqrt{\|\E[\bM] \|_1 \|\E[\bM] \|_\infty} \le  \frac{8 \ell^3 r^2 \sqrt{km}}{D^{1.5} \alpha^2}.$$ 
\item Let $\bM$ be a matrix such that  $[\bM]_{i, j} = b_x \phi_{\alpha}(\langle [\bE]_i, x \rangle)  \phi_{\alpha}(\langle [\bE']_j, x \rangle)$, then
$$\|\E[\bM] \|_2  \le  \sqrt{\|\E[\bM] \|_1 \|\E[\bM] \|_\infty} \le \frac{32 \ell^6 r^4 m}{\alpha^4 D^2}.$$ 
\end{enumerate}

\end{lem}

\begin{proof}[Proof of Lemma \ref{lem:exp_M1}]

Since all the $ \phi_{\alpha}(\langle [\bE]_i, x \rangle) $ and $x_i$ are non-negative, without lose of generality we can assume that $b_x  = 1$. 

\begin{enumerate}

%1
\item We have
$$ \sum_{j \in [D]} \E[\bM_{i, j}] = \E\left[\phi_{\alpha}(\langle [\bE]_i, x \rangle) \sum_{j \in [D]}  x_j\right] \le  r \E[\phi_{\alpha}(\langle [\bE]_i, x \rangle) ] \le \frac{16 \ell^4 r^3 m}{\alpha^2(\alpha - 2 \ell) D^2}.$$

On the other hand, 
\begin{align*} \sum_{i \in [D]} \E[\bM_{i, j}] = \E\left[\left(\sum_{i \in [D]} \phi_{\alpha}(\langle [\bE]_i, x \rangle) \right)  x_j\right] \le \E[ (u_x \bE x ) x_j]
\end{align*}
where $u_x$ is a vector with each entry either $0$ or $1$ depend on $\langle [\bE]_i, x \rangle \ge \alpha$ or not. Note that $\sum_i \langle [\bE]_i, x \rangle^2  \le  \ell^2 r$, so $u_x$ can only have at most $\frac{\ell^2 r}{\alpha^2}$ entries equal to $1$, so $\| u_x \|_2 \le \frac{\ell \sqrt{r}}{\alpha}$. This implies that 
$$(u_x \bE x )  \le\frac{\ell^2 r}{\alpha}.$$

Therefore, $ \sum_{i \in [D]} \E[\bM_{i, j}] \le \frac{2 \ell^2 r k}{\alpha D} $, which implies that 
$$\| \E[\bM]  \|_2 \le \frac{4 \sqrt{2} \ell^3 r^2 \sqrt{km}}{D^{1.5} \alpha^{1.5 }\sqrt{ (\alpha - 2 \ell)}}.$$

%2
\item We have
$$\|\E[ \bM]\|_1  \le \max_i \sum_{j \in [D]} \E[\bM_{i, j}] = \max_i\E\left[\phi_{\alpha}(\langle [\bE]_i, x \rangle) \sum_{j \in [D]} \phi_{\alpha}(\langle [\bE']_j, x \rangle)  \right] \le  \frac{\ell^2 r}{\alpha} \frac{16 \ell^4 r^3 m}{D^2 \alpha^2(\alpha - 2\ell)}.$$

In the same way we can bound $\|\E[ \bM]\|_{\infty}$ and get the desired result.

\end{enumerate}
\end{proof}

\subsubsection{Lemmas with $x_i \in \{0, 1\}$ }

Here we present some expectation lemmas when $x_i \in \{0, 1\}$.

\begin{lem}[Expectation, Matrix] \label{lem:exp_M2}

For every $\ell, \ell' \in [0, 1)$, every matrices $\bE, \bE' \in \mathbb{R}^{D \times D}$ such that $\| \bE \|_2, \| \bE' \|_2 \le \ell$, and $\forall i \in [D], |\bE_{i, i}|  |\bE_{i, i}'|\le \ell'$, then for every $\beta > 4\ell'$ and  $\alpha \ge 4 \ell$ and every $b_x \in [-1, 1]$ that depends on $x$, the following hold.
\begin{enumerate}
\item Let $\bM$ be a matrix such that $[\bM]_{i, j} \sqrt{\|\E[\bM] \|_1 \|\E[\bM] \|_\infty}= b_x \phi_{\beta}(\langle [\bE]_i, x \rangle) x_i x_j$, then
$$\|\E[\bM] \|_2 \le \sqrt{\|\E[\bM] \|_1 \|\E[\bM] \|_\infty} \le \frac{8 \ell^{1.5} r^{1.25}m}{D^{1.5} \beta^{0.5}}.$$ 
\item Let $\bM$ be a matrix such that  $[\bM]_{i, j} = b_x \phi_{\beta}(\langle [\bE]_i, x \rangle)  x_i x_j \phi_{\beta}(\langle [\bE']_j, x \rangle) $, then
$$\|\E[\bM] \|_2 \le \sqrt{\|\E[\bM] \|_1 \|\E[\bM] \|_\infty} \le \frac{8 \ell^4 r^3 m}{\beta^2 D^2}.$$ 
\item Let $\bM$ be a matrix such that  $[\bM]_{i, j} = b_x \phi_{\beta}(\langle [\bE]_i, x \rangle) x_i \phi_{\alpha}(\langle [\bE']_j, x \rangle) $, then
$$\|\E[\bM] \|_2 \le \sqrt{\|\E[\bM] \|_1 \|\E[\bM] \|_\infty} \le \frac{16 \ell^5 r^{2.5} m}{D^2 \alpha^2 \beta}.$$
\end{enumerate}

\end{lem}

\begin{proof}
This Lemma is a special case of Lemma \ref{lem:exp_M3} by setting $\gamma = 1$.
\end{proof}

\subsubsection{Lemmas with general $x_i$}

Here we present some expectation lemmas for the general case where $x_i \in [0, 1]$ and the distribution of $x$ satisfies the order-$q$ decay condition. 

\begin{lem}[General expectation]\label{lem:ge}
Suppose the distribution of $x$ satisfies the order-$q$ decay condition. 
$$\forall i \in [D], \quad \E[x_i] \le  \Pr[x_i \not= 0] \le \frac{(q + 2)2k}{qD}.$$
\end{lem}

\begin{proof}
Denote $s = \Pr[x_i \not= 0]$. By assumption, $\Pr[x_i \le \sqrt{\alpha} \mid x_i \not=0] \le \alpha^{q/2}$, which implies that  $\Pr[x_i > \sqrt{\alpha} ] >  s(1 - \alpha^{q/2})$.
Now, since $$\E[x_i^2]  = \int_{0}^1  \Pr[x_i^2 \ge \alpha] = \int_{0}^1  \Pr[x_i \ge \sqrt{\alpha}] \le \frac{2k}{D},$$

We obtain $$s \le \frac{2k}{D} \frac{1}{\int_{0}^1  (1 - \alpha^{q/2} )d \alpha } \le \frac{q + 2}{q } \frac{2k}{D}.$$

\end{proof}

\begin{lem}[Truncated covariance]\label{lem:tc}
For every $\alpha > 0$, every $b_x \in [-1, 1]$ that depends on $x$, the following holds.
Let $\bM$ be a matrix such that $[\bM]_{i, j} = b_x \mathsf{1}_{x_i \le \alpha} x_i x_j$, then
$$\|\E[\bM] \|_2 \le \sqrt{\|\E[\bM] \|_1 \|\E[\bM] \|_{\infty} }  \le \frac{6kr}{D} \alpha^{\frac{q + 1}{2}}.$$ 
\end{lem}

\begin{proof}[Proof of Lemma \ref{lem:tc}]
Again, without lose of generality we can assume that $b_x$ are just $1$. 

On one hand, 
$$\sum_{j \in [D]} \E[\mathsf{1}_{x_i \le \alpha} x_i x_j] \le r \E[\mathsf{1}_{x_i \le \alpha} x_i ] \le r\alpha \E[\mathsf{1}_{0 < x_i \le \alpha} ]  =  r \alpha \Pr[x_i \in (0, \alpha]].$$

By Lemma \ref{lem:ge},

$$ \Pr[x_i \in (0, \alpha]] = \Pr[x_i \not= 0] \Pr[x_i \le \alpha \mid x_i \not=0] \le  \frac{(q + 2)2k}{qD} \alpha^q,$$ 

and thus
$$ \sum_{j \in [D]} \E[\mathsf{1}_{x_i \le \alpha} x_i x_j]  \le \frac{2(q + 2)kr}{q D} \alpha^{q + 1} \le \frac{6kr}{ D} \alpha^{q + 1}.$$

On the other hand, 
$$\sum_{i  \in [D]} \E[\mathsf{1}_{x_i \le \alpha} x_i x_j] \le \E\left[\left(\sum_{i  \in [D]} x_i\right) x_j \right]  \le \frac{6kr}{D}.$$

Putting everything together we completes the proof.
\end{proof}

\begin{lem}[Truncated half covariance]\label{lem:tch}
For every $\alpha > 0$, every $b_x \in [-1, 1]$ that depends on $x$, the following holds.
Let $\bM$ be a matrix such that $[\bM]_{i, j} = b_x \mathsf{1}_{x_i \ge \alpha}  x_j$, then
$$\|\E[\bM] \|_2 \le \sqrt{\|\E[\bM] \|_1 \|\E[\bM] \|_{\infty} }  \le 12k \left(\frac{r}{D}\right)^{\frac{2q + 1}{2q + 2}} $$ 
\end{lem}

\begin{proof}[Proof of Lemma \ref{lem:tch}]
Without lose of generality, we can assume $b_x = 1$. %Now we just need to show that  for every $s \in (0, 1]$, $$\E[\mathsf{1}_{x_i \ge \alpha}  x_j] \le  \frac{1}{s}\E[x_i x_j]  +   s^q \frac{6k }{D}$$
We know that 
\begin{eqnarray*}
\E[\mathsf{1}_{x_i \ge \alpha}  x_j]  &\le& \Pr[x_i \not = 0] \E[x_j \mid x_i \not= 0] 
\\
&\le& \frac{1}{s} \Pr[x_i \not = 0] \E[  \mathsf{1}_{x_i  \ge s} x_i x_j \mid x_i \not= 0]  +  \Pr[x_i \not = 0] \E[  \mathsf{1}_{x_i  < s}  \mid x_i \not= 0]  
\\
&\le& \frac{1}{s}\E[x_i x_j] + s^q \Pr[x_i \not= 0].
\end{eqnarray*}

From Lemma \ref{lem:ge} we know that $\Pr[x_i \not= 0] \le \frac{6k }{D}$, which implies that 
$$
  \sum_{i \in [D]}\E[x_i x_j]  = \E[\sum_{i \in [D]} x_i x_j] \le r \E[x_j] \le r \Pr[x_j \ne 0] \le \frac{6kr}{D}.
$$

Therefore,  
$$\sum_{i \in [D]} \E[\mathsf{1}_{x_i \ge \alpha}  x_j]  \le 6k s^q + \frac{1}{s} \frac{6kr}{D} $$

Choosing the optimal $s$, we are able to obtain
$$\sum_{i \in [D]} \E[\mathsf{1}_{x_i \ge \alpha}  x_j]  \le \left(\frac{r}{qD} \right)^{q/(q + 1)} 6k (q + 1) \le 24 k \left(\frac{r}{D} \right)^{q/(q + 1)}.$$

On the other hand, 
$$\sum_{j \in [D]} \E[\mathsf{1}_{x_i \ge \alpha}  x_j]   \le r  \E[\mathsf{1}_{x_i > 0}] \le \frac{6k}{D} r.$$

Putting everything together we get the desired bound. 
\end{proof}

\begin{lem}[Expectation]\label{lem:exp3}
For every $\ell \in [0, 1)$, every vector $e$ such that $\| e \|_2 \le \ell$, for every $i \in [D], \alpha > 2 | e_i |, \gamma > 0$, the following hold.
\begin{enumerate}
\item $$\forall i \in [D]: \ \E[\phi_{\alpha } (\langle e, x \rangle ) \mathsf{1}_{x_i \ge \gamma} ] \le  \frac{4 \ell^2 r m }{ \gamma D^2 (\alpha - 2 |e_i|) }.$$
\item If $e_i = 0$, then $$\forall i \in [D]: \ \E[\phi_{\alpha } (\langle e, x \rangle ) \mathsf{1}_{x_i \ge \gamma} ] \le \frac{24 k r \ell^2}{D \alpha}  \left( \frac{m}{D k } \right)^{q/(q + 1)}.$$
\end{enumerate}
\end{lem}

\begin{proof}[Proof of Lemma \ref{lem:exp3}]

We define $g$ as in Lemma \ref{lem:exp1}.  We still have $\| g \|_1 \le \frac{\| g\|_2^2}{\frac{\alpha}{2r}} \le  \frac{2r \ell^2}{\alpha}$.

\begin{enumerate}
	\item The value $\phi_{\alpha} (\langle e, x \rangle) \mathsf{1}_{x_i \ge \gamma}  $ is non-zero only when $x_i  \geq \gamma $. Therefore, we shall only focus on this case. 

Let us  again suppose $x$ such that $\langle g, x \rangle \ge \frac{\alpha}{2}$ and $x_i \ge  \gamma$  forms a set $\set{S}$ of size $S$, each $x^{(s)} \in \set{S}$ has probability $p_s$. 

\begin{claim}\label{cla:exp3}
\begin{enumerate}
\item[(1)] $\forall s \in [S]: \quad  \sum_{j \in \set{E}} g_j x^{(s)}_j \ge \frac{\alpha}{2} $.  \label{item:shogasafaisas}
\item[(2)] $\forall j \not= i \in [D]: \  \sum_{s \in [S]} p_s x^{(s)}_j x^{(s)}_i \le \frac{m}{D^2}$. \label{item:pqrqpowiruwqjfhsoafj}
\item[(3)] By Lemma \ref{lem:tch}, $$\forall j \not= i \in [D]: \  \sum_{a \in [S]} p_a x^{(a)}_j \le \frac{1}{s} \frac{m}{D^2} +    s^q \frac{6k }{D}  =  \frac{6k (q + 1)}{D} \left( \frac{m}{6D k q} \right)^{q/(q + 1)}$$ by choosing optimal $s$.
Moreover, we can directly calculate that $$\frac{6k (q + 1)}{D} \left( \frac{m}{6D k q} \right)^{q/(q + 1)} \le \frac{6k}{D}  \left( \frac{m}{D k } \right)^{q/(q + 1)}.$$ \label{item:hgaoghusafja}
\end{enumerate}
\end{claim}

With Claim~\ref{cla:exp3}(2), multiply both side by $g_j$ and taking the summation,
$$  \sum_{s \in [S], j \in \set{E}, j \not= i } p_s g_j x^{(s)}_j  x^{(s)}_i \le \frac{m}{D^2} \sum_{j \in \set{E} , j \not= i } g_j.$$ 

Using the fact that $x^{(s)}_i \ge \gamma$ for every $s \in [S]$, we obtain 
$$  \sum_{s \in [S], j \in \set{E}, j \not= i } p_s g_j x^{(s)}_j  \le \frac{m}{ \gamma D^2} \sum_{j \in \set{E} , j \not= i } g_j.$$

On the other hand, by Claim~\ref{cla:exp3}(1) and the fact that $|e_i| \ge g_i \ge 0$, we know that 
$$\sum_{j \in \set{E}, j \not= i} g_j x^{(s)}_j \ge \frac{\alpha}{2} - |e_i|.$$ 

Using the fact that $x^{(s)}_i \ge \gamma$ for every $s \in [S]$, we obtain 
$$  \sum_{s \in [S]} p_s \le \frac{2  m }{ \gamma (\alpha - 2 |e_i| )D^2}  \sum_{j \in \set{E} , j \not= i } g_j.$$

Therefore, since $g_i \le |e_i|$,
\begin{align*}
 \sum_{s \in [S], j \in \set{E}} p_s g_j x^{(s)}_j & \le \sum_{s \in [S], j \in \set{E}, j\neq i} p_s g_j x^{(s)}_j   +  \sum_{s \in [S]} p_s g_i x^{(s)}_i \\
& \le   \frac{m}{\gamma D^2} \left( 1 + \frac{2 g_i }{\alpha - 2 |e_i|}  \right)\sum_{j \in \set{E} , j \not= i } g_j \le    \frac{m}{\gamma D^2} \frac{\alpha}{\alpha - 2 |e_i|} \| g \|_1  \le \frac{ 2 \ell^2 r m }{\gamma D^2 (\alpha - 2 |e_i|)}.
\end{align*}

\item When $e_i = 0$, in the same manner, but using Claim~\ref{cla:exp3}(3), we obtain
$$\sum_{s \in [S]} p_s g_j x^{(s)}_j \le \frac{6k}{D}  \left( \frac{m}{D k } \right)^{q/(q + 1)}  g_j.$$

Summing over $j \in \set{E}, j \not= i$ we have:
$$ \sum_{s \in [S], j \in \set{E}, j \not= i} p_s g_j x^{(s)}_j \le \frac{6k}{D}  \left( \frac{m}{D k } \right)^{q/(q + 1)} \frac{2 r \ell^2}{\alpha}.$$
\end{enumerate}
\end{proof}

\begin{lem}[Expectation, Matrix] \label{lem:exp_M3}

For every $\ell, \ell' \in [0, 1)$, every matrices $\bE, \bE' \in \mathbb{R}^{D \times D}$ such that $\| \bE \|_2, \| \bE' \|_2 \le \ell$, and $\forall i \in [D], |\bE_{i, i}|, |\bE_{i, i}'|\le \ell'$, every $\beta > 4\ell'$ and  $\alpha \ge 4 \ell$, every $\gamma > 0$ and every $b_x \in [-1, 1]$ that depends on $x$, the following hold.
\begin{enumerate}
\item Let $\bM$ be a matrix such that $[\bM]_{i, j} = b_x \phi_{\beta}(\langle [\bE]_i, x \rangle) \mathsf{1}_{x_i \ge \gamma} x_j$, then
$$\|\E[\bM] \|_2 \le \sqrt{\|\E[\bM] \|_1 \|\E[\bM] \|_{\infty} }  \le \min\left\{ \frac{8 \ell^{2}\sqrt{k} r^{1.5} \sqrt{m}}{\sqrt{\gamma} D^{1.5} \beta } ,  \frac{12k \ell^2 r}{D \beta}\left( \frac{m}{D k} \right)^{q/(2q + 2)}  \right\}.$$ 
\item Let $\bM$ be a matrix such that  $[\bM]_{i, j} = b_x \phi_{\beta}(\langle [\bE]_i, x \rangle)  \mathsf{1}_{x_i \ge \gamma} \mathsf{1}_{x_j \ge \gamma} \phi_{\beta}(\langle [\bE']_j, x \rangle) $, then
$$\|\E[\bM] \|_2 \le \sqrt{\|\E[\bM] \|_1 \|\E[\bM] \|_{\infty} }  \le \frac{8 \ell^4 r^2 m}{\gamma \beta^2 D^2}.$$ 
\item Let $\bM$ be a matrix such that  $[\bM]_{i, j} = b_x \phi_{\beta}(\langle [\bE]_i, x \rangle) \mathsf{1}_{x_i \ge \gamma} \phi_{\alpha}(\langle [\bE']_j, x \rangle) $, then
$$\|\E[\bM] \|_2 \le \sqrt{\|\E[\bM] \|_1 \|\E[\bM] \|_{\infty} } \le \frac{16 \ell^5 r^{2.5} m}{\sqrt{\gamma }D^2 \alpha^2 \beta}.$$
\end{enumerate}
\end{lem}

\begin{proof}[Proof of Lemma \ref{lem:exp_M3}]
Without loss of generality, assume $b_x = 1$.

\begin{enumerate}
\item  
Since every entry of $\bM$ is non-negative, by Lemma \ref{lem:exp3}, 
$$ \sum_{j \in [D]} \E[\bM_{i, j}] = \E\left[\phi_{\beta}(\langle [\bE]_i, x \rangle) \mathsf{1}_{x_i \ge \gamma} \sum_{j \in [D]}   x_j\right] \le  r \E[\phi_{\beta}(\langle [\bE]_i, x \rangle)  \mathsf{1}_{x_i \ge \gamma} ] \le \frac{4 \ell^2 r^2 m}{\gamma D^2 (\beta - 2\ell')}$$

and 
$$ \sum_{j \in [D]} \E[\bM_{i, j}] \le  \frac{24 k r \ell^2 }{D\beta} \left( \frac{m}{D k } \right)^{q/(q + 1)}.$$

%\yingyu{const 8 should be 16? since $\beta - 2 |e_i| \ge \beta /2$}

On the other hand, as in Lemma \ref{lem:exp_M1}, we know that 
$$\sum_{j \in [D]} \phi_{\beta}(\langle [\bE']_j, x \rangle) \le \frac{\ell^2 r}{\beta}.$$
 
 Therefore, 
$$\sum_{i \in [D]}\E[\phi_{\beta}(\langle [\bE]_i, x \rangle)  \mathsf{1}_{x_i \ge \gamma} x_j ] \le \frac{\ell^2 r}{\beta} \E[x_j] \le \frac{6k \ell^2 r}{\beta D}.$$

%$$\E[\phi_{\beta}(\langle [\bE]_i, x \rangle)  \mathsf{1}_{x_i \ge \gamma} x_j ] \le \E[ \|  [\bE]_i \|_2 \|x\|_2  \mathsf{1}_{x_i \ge \gamma} x_j ] \le  \ell \sqrt{r} \E[ \mathsf{1}_{x_i \ge \gamma} x_j ] \le  \ell \sqrt{r} \frac{1}{\gamma}\E[ x_i x_j ] \le \frac{ \ell \sqrt{r} m}{\gamma D^2}$$

Now, since each entry of $\bM$ is non-negative, using $\| \E[\bM]  \|_2 \le \sqrt{\| \E[\bM]  \|_{1} \| \E[\bM]  \|_{\infty}}$, we obtain the desired bound. 
%$$\| \E[\bM]  \|_2 \le \sqrt{\| \E[\bM]  \|_{1} \| \E[\bM]  \|_{\infty}} \le \frac{2 \sqrt{2}\sqrt{k} \ell^{2} r^{1.5} \sqrt{m}}{\sqrt{\gamma} D^{1.5} \sqrt{\beta }(\beta - 2\ell')^{0.5}} $$
%\yingyu{ added a $\sqrt{k}$}

%$$\| \E[\bM]  \|_2 \leq \frac{8k \ell^2 r \sqrt{q + 1}}{D \beta}\left( \frac{m}{4D k q} \right)^{q/(2q + 2)} $$

%2
\item

Since now $\bM$ is a ``symmetric'' matrix, we only need to look at $\sum_{j \in [D]} \E[\bM_{i, j}]$, and a similar bound holds for $\sum_{i \in [D]} \E[\bM_{i, j}]$.
$$ \sum_{j \in [D]} \E[\bM_{i, j}] = \E\left[\phi_{\beta}(\langle [\bE]_i, x \rangle)  \mathsf{1}_{x_i \ge \gamma} \sum_{j \in [D]} \phi_{\beta}(\langle [\bE']_j, x \rangle)   \mathsf{1}_{x_j \ge \gamma}\right] \le  \frac{\ell^2 r}{\beta} \frac{4 \ell^2 r m}{\gamma D^2 (\beta - 2\ell')}.$$
\yingyu{an extra $r$ on the right hand side?}\Ynote{yes, fixed}
The conclusion then follows.

%3
\item 
On one hand, 
$$ \sum_{j \in [D]} \E[\bM_{i, j}] = \E\left[\phi_{\beta}(\langle [\bE]_i, x \rangle)  \mathsf{1}_{x_i \ge \gamma} \sum_{j \in [D]} \phi_{\alpha}(\langle [\bE']_j, x \rangle)  \right] \le  \frac{\ell^2 r}{\alpha} \frac{4 \ell^2 r m}{\gamma D^2 (\beta - 2 \ell')}.$$

On the other hand, 
$$ \sum_{i \in [D]} \E[\bM_{i, j}] = \E\left[ \left( \sum_{i \in [D]} \phi_{\beta}(\langle [\bE]_i, x \rangle)  \mathsf{1}_{x_i \ge \gamma} \right) \phi_{\alpha}(\langle [\bE']_j, x \rangle)  \right] \le  \frac{\ell^2 r}{\beta} \frac{16 \ell^4 r^2 m}{D^2 \alpha^2(\alpha - 2 \ell)}.$$

Therefore, 
$$\| \E[\bM ]\|_2 \le \frac{8 \ell^5 r^{2.5} m}{  \sqrt{\gamma} D^2 \alpha^{1.5} \beta^{0.5} \sqrt{\beta - 2 \ell'} \sqrt{\alpha - 2 \ell}} \le \frac{16 \ell^5 r^{2.5} m}{\sqrt{\gamma }D^2 \alpha^2 \beta}.$$

\end{enumerate}
\end{proof}

\subsection{Robustness} \label{app:noise}

In this subsection, we show that our algorithm is also robust to noise. To demonstrate the idea, we will present a proof for the case when $x_i \in \{0, 1\}$. 
The general case when $x_i \in [0,1]$ follows from the same argument, just with more calculations.

\begin{lem}[Expectation]\label{lem:expn}
For every $\ell, \nu \in [0, 1)$, every vector $e$ such that $\| e \|_2 \le \ell$, every $\alpha$ such that $\alpha > 2  \ell + 2 \nu$, the following hold. 
\begin{enumerate}
\item $\E[ \phi_{\alpha } (\langle e, x \rangle + \nu ) ] \le \frac{16 m \ell^4 r^2}{\alpha^2(\alpha - 2 \ell - 2 \nu) D^2}$. 
\item If $e_{i, i} = 0$, then $\E[| \langle e_i, x \rangle | x_i] \le  \sqrt{\frac{2mkr}{D^3}}.$
\end{enumerate}
\end{lem}

\begin{proof}[Proof of Lemma \ref{lem:expn}]

The proof of this lemma is almost the same as the proof of Lemma \ref{lem:exp1} with a few modifications. 

1. Without lose of generality, we can assume that all the entries of $e$ are non-negative. Let us denote a new vector $g$ such that 
$$g_i =  \left\{ \begin{array}{ll}
         e_i & \mbox{if $e_i \geq \frac{\alpha}{2r}$},\\
        0 & \mbox{otherwise}.\end{array} \right.$$

Due to the fact that $\| x\|_1 \le r$, we can conclude $\langle e - g, x \rangle \le \frac{\alpha}{2 r} \times r = \frac{\alpha}{2}$, which implies
$$\phi_{\frac{\alpha}{2}} (\langle g, x\rangle + \nu ) \ge \frac{1}{2} \phi_{\alpha} (\langle e, x \rangle + \nu).$$

Now we can only focus on $g$. Since $\| g\|_2 \le \ell$,  we know that $g$ has at most $\frac{4 \ell^2 r^2}{\alpha^2}$ non-zero entries.  Let us then denote the set of non-zero entries of $g$ as $\set{E}$. Then we have $| \set{E} | \le \frac{4 \ell^2 r^2}{\alpha^2} $.

Suppose the all the $x$ such that $\langle g, x \rangle \ge \frac{\alpha}{2} - \nu$ forms a set $\set{S}$ of size $S$, each $x^{(s)} \in \set{S}$ has probability $p_t$. Then 
$$\E[ \phi_{\alpha } (\langle e, x \rangle ) ] \le 2 \sum_{s \in [S]} p_s \langle g, x^{(s)}  \rangle = 2 \sum_{s \in [S], i \in \set{E}} p_s g_i x^{(s)}_i.$$

On the other hand, we have the following claim. 
\begin{claim}
\begin{enumerate}
\item $\forall s \in [S]: \quad  \sum_{i \in \set{E}} g_i x^{(s)}_i \ge \frac{\alpha}{2}  - \nu$.  \label{item:vhaofnqiejglkjafajkfhasfjka2}
\item $\forall i \not= j \in [D]: \  \sum_{s \in [S]} p_s x^{(s)}_i x^{(s)}_j \le \frac{m}{D^2}$. This is by the GCC conditions of the distribution of $x$. 
\label{item:ahfoadghiafjhqiofqoipfjaofi2}
\end{enumerate}
\end{claim}

Using $(\ref{item:ahfoadghiafjhqiofqoipfjaofi2})$ and multiply both side by $g_i g_j$, we get
$$\sum_{s \in [S]} p_s (g_i x^{(s)}_i ) (g_j x^{(s)}_j) \le \frac{m g_i g_j}{D^2 }.$$

Sum over all $j \in \set{E}, j \not= i$,
$$\sum_{s \in [S]} \sum_{j \in \set{E}, j \not= i } p_s (g_i x^{(s)}_i ) (g_j x^{(s)}_j) \le \frac{m g_i}{D^2 } \left(\sum_{j \in \set{E}, j \not=i }g_j \right) \le \frac{m g_i}{D^2 } \sum_{j \in \set{E} }g_j.$$

Using $(\ref{item:vhaofnqiejglkjafajkfhasfjka2})$, and that $\sum_{j \in \set{E}} g_j x^{(s)}_j \ge \frac{\alpha}{2}  - \nu$ and $g_i \le \ell, x^{(s)}_i \le 1$, we can obtain
$$\sum_{j \in \set{E}, j \not= i} g_j x^{(s)}_j \ge \frac{\alpha}{2} -  \nu - \ell.$$ 

This implies 
$$\sum_{s \in [S]} p_s (g_i x^{(s)}_i ) \le \frac{1}{ \sum_{j \in \set{E}, j \not= i } g_j x^{(s)}_j } \left( \frac{m g_i}{D^2 } \sum_{j \in \set{E} }g_j   \right) \le \frac{2 m}{(\alpha - 2 \nu - 2 \ell )D^2} g_i \sum_{j \in \set{E} } g_j.$$

Summing over $i$, 
$$\sum_{s \in [S], i \in \set{E}} p_s g_i x^{(s)}_i  \le \frac{2 m}{(\alpha - 4 \ell )D^2} \left(\sum_{j \in \set{E} } g_j \right)^2 \le  \frac{2 m}{(\alpha - 2 \nu - 2\ell )D^2} | \set{E} | \| g\|_2^2 \le \frac{8 m \ell^4   r^2}{\alpha^2(\alpha - 2 \ell - 2 \nu) D^2}.$$

2. We can directly bound this term as follows.
$$\E[| \langle e , x \rangle | x_i ] \le  \sum_{j \not = i}| e_j |\E[x_{i} x_j ] \le \ell \sqrt{ \sum_{j \not= i} \E[x_i x_j]^2 } \le \ell \sqrt{  \frac{m}{D^2}\sum_{j \not= i}  \E[x_i x_j]^2   } \le \ell \sqrt{\frac{2mkr}{D^3}}.$$
\end{proof}

We show the following lemma saying that even with noise, $\bA^{\dagger} \bAg$   is roughly  $(\bSigma + \bE)^{-1} $. 
\begin{lem}[Noisy inverse]\label{lem:noisy_inv}
Let $\bA \in \mathbb{R}^{W \times D}$ be a matrix such that $\bA = \bAg(\bSigma + \bE) + \bNoise$, for diagonal matrix $\bSigma \succeq \frac{1}{2} \bI$, off diagonal matrix $\bE$ with $\| \bE \|_2 \le \ell \le \frac{1}{8}$ and $\|\bNoise \|_2 \le \frac{1}{4} \sigma_{\min} (\bAg)$. Then 
$$\| \bA^{\dagger} \bAg   - (\bSigma + \bE)^{-1} \|_2 \le   \frac{2 \| \bNoise \|_2}{\left(\frac{1}{2} - \frac{3}{2} \ell \right)  \sigma_{\min} (\bAg) - \| \bNoise \|_2} \le \frac{32 \| \bNoise \|_2}{ \sigma_{\min} (\bAg)}.$$
\end{lem}

\begin{proof}[Proof of Lemma \ref{lem:noisy_inv}]
\begin{eqnarray*}
\| \bA^{\dagger} \bAg  - (\bSigma + \bE)^{-1} \|_2 &\le& \| \bA^{\dagger} (\bAg (\bSigma + \bE) + \bNoise)  (\bSigma + \bE)^{-1}  - (\bSigma + \bE)^{-1} \|_2 + \|  \bA^{\dagger} \bNoise\|_2  \| (\bSigma + \bE)^{-1} \|_2 
\\
&\le& \|  \bA^{\dagger} \bNoise\|_2  \| (\bSigma + \bE)^{-1} \|
\\
&\le& \frac{2}{(1 - \ell)\sigma_{\min} (\bA)} \|\bNoise\|_2.
\end{eqnarray*}

Since $\bA = \bAg(\bSigma + \bE) + \bNoise$, 
$$\sigma_{\min}  (\bA) \ge \sigma_{\min} (\bAg(\bSigma + \bE) ) - \| \bNoise \|_2 \ge \left(\frac{1}{2} - \ell \right)\sigma_{\min} (\bAg)  - \| \bNoise \|_2.$$

Putting everything together, we are able to obtain
$$\| \bA^{\dagger} \bAg  - (\bSigma + \bE)^{-1} \|_2 \le  \frac{2}{(1 - \ell)  \left(\frac{1}{2} - \ell \right) \sigma_{\min} (\bAg)  - \| \bNoise \|_2} \|\bNoise\|_2  \le \frac{2 \| \bNoise \|_2}{\left(\frac{1}{2} - \frac{3}{2} \ell \right)  \sigma_{\min} (\bAg) - \| \bNoise \|_2}.$$
\end{proof}

\begin{lem}[Noisy decoding] \label{lem:noisy_decoding}

Suppose we have $z = \phi_{\alpha}( (\bSigma' + \bE' ) x + \xi^x)  $ for diagonal matrix $\|\bSigma' - \bI \|_2 \le \frac{1}{2}$ and off diagonal matrix $\bE'$ such that $\| \bE' \|_2 \le \ell \le \frac{1}{8}$ and   random variable $\xi^x$ depend on $x$ such that $\|\xi^x\|_{\infty} \le \nu$. Then if $\frac{1}{4} > \alpha > 4 \ell + 4 \nu, m \le \frac{D}{r^2}$, we  have

 $$\| \E[(\bSigma x - z)x^{\top} ]\|_2 , \| \E[(\bSigma x - z)z^{\top} ]\|_2 = O\left( C_3 \right) $$

 where $$C_3 =  (\nu + \beta)  \frac{kr}{D}+ \frac{m \ell^4 r^2}{\alpha^3 D^2} + \frac{\ell^2 \sqrt{km} r^{1.5}}{D^{1.5} \beta } + \frac{\ell^4 r^3 m}{\beta^2 D^2} + \frac{\ell^5 r^{2.5} m}{D^2 \alpha^2 \beta}.$$
\end{lem}

\begin{proof}[Proof of Lemma \ref{lem:noisy_decoding}]
Since we have now
$$z_i = \phi_{\alpha} ( \bSigma_{i, i}' x_i + \langle e_i, x \rangle + \xi^x_i).$$

Like in Lemma \ref{lem:decoding1}, we can still show that 
$$| \bSigma_{i, i}' x_i + \langle e_i, x \rangle x_i + \xi^x_i x_i  - z_i | \le  \phi_{\alpha}(\langle e_i, x \rangle + \xi^x_i) \le   \phi_{\alpha}(\langle e_i, x \rangle + \nu)$$

which implies that there exists $a_{x, \xi}\in [-1, 1]$ that depends on $x, \xi$ such that 
$$ z_i   - \bSigma_{i, i}' x_i  = \langle e_i, x \rangle x_i + \xi^x_i x_i + a_{x,  \xi}  \phi_{\alpha}(\langle e_i, x \rangle + \nu).$$

Therefore, 
\begin{eqnarray*}
\E[z_i^2 ] &\le& 3 (\bSigma_{i, i}'  +\nu)^2\E[x_i^2] +  3\E[ \langle e_i, x \rangle^2 x_i^2] + 3\E[ \phi_{\alpha}(\langle e_i, x \rangle + \nu)^2 ] 
\\
&\le&\frac{6 (2 + \nu)^2 k}{D} + 3 \ell^2 r \sqrt{\frac{2mk}{D^3}} + \frac{48 (\ell \sqrt{r} + \nu)m \ell^4 r^2}{\alpha^2 (\alpha - 2 \nu - 2 \ell)D^2} 
\\
&\le& O\left( \frac{ k }{D} \right).
\end{eqnarray*}

Again, from $ z_i  - \bSigma_{i, i}' x_i   = \langle e_i, x \rangle x_i + \xi^x_i x_i +a_{x,  \xi}  \phi_{\alpha}(\langle e_i, x \rangle + \nu)$, following the exact same calculation as in Lemma \ref{lem:decoding1}, but using Lemma \ref{lem:expn} instead of Lemma \ref{lem:exp1}, we obtain the result.
\end{proof}

\newcommand{\btN}{\widetilde{\bold{N}}}

\begin{defn}[$(\gamma_1, \gamma_2)$-rounded]
A random variable $\zeta$ is $(\gamma_1, \gamma_2)$ rounded if 
$$\|\E[\zeta \zeta^{\top}] \|_2 \le \gamma_1, \quad \| \zeta\|_2 \le \gamma_2.$$
\end{defn}

\begin{thm}[Noise] \label{thm:main_noise}
Suppose $\bA_0$ is $(\ell, \rho)$-initialization for $\ell = O(1), \rho = O(\sigma_{\min} (\bAg))$.
Suppose that the data is generated from $y^{(t)} = \bAg x^{(t)} + \zeta^{(t)}$, where $ \zeta^{(t)} $ is $(\gamma_1, \gamma_2)$-rounded, and $\gamma_2 = O(\sigma_{\min} (\bAg))$. \yingyu{I think we need this bound on the size of the noise? }
%Suppose that the decoding at round $t$ is given by: 
%$$z^{(t)} = \phi_{\alpha} ( \bA_0^{\dagger} ( \bAg x^{(t)} + \zeta^{(t)} ))$$

Then after $\poly(D, \frac{1}{\epsilon})$ iterations, Algorithm~\ref{alg:and} outputs a matrix $\bA$ such that there exists diagonal matrix $\tilde\bSigma \succeq \frac{1}{2} \bI$ with
$$ \| \bA - \bAg \tilde\bSigma \|_2 = O\left( r \frac{\gamma_2 }{\lambda } \frac{\sigma_{\max} (\bAg)}{\sigma_{\min} (\bAg)}  + \frac{\sqrt{ \gamma_1} }{  \lambda }  \sqrt{\frac{D}{k}} + \veps \right).$$
\end{thm}

\Ynote{the final theorem is a typo and there should be a $\lambda$. Second term is always smaller since for gaussian $\gamma_2^2 = D \gamma_1$. }
\yingyu{How does this connect to the theorem? In the theorem, the second term is replaced by an $\epsilon$ and no $\lambda$ in the first.}

\begin{proof}[Proof of Theorem~\ref{thm:main_noise}]

For notation simplicity, we only consider one stages, and we drop the round number here and let $\btA = \bA^{(t + 1)}  $ and $ \bA = \bA^{(t)} $, and we denote the new decomposition as $\btA = \bAg (\btSigma + \btE) + \btN$.

Thus, the decoding of $z$ is given by
\begin{eqnarray*}
z &=& \phi_{\alpha} ( \bA_0^{\dagger} ( \bAg x + \zeta )).
\end{eqnarray*}

By Lemma \ref{lem:noisy_inv}, there exists a matrix $\bR $ such that $\| \bR \|_2 \le \frac{32 \| \bNoise_0 \|_2}{\sigma_{\min} (\bAg)}$ with $\bA_0^{\dagger}  \bAg = (\bSigma_0 + \bE_0)^{-1} + \bR$. 
Now let $\bSigma' + \bE' = (\bSigma_0 + \bE_0)^{-1} + \bR$, where $\bSigma'$ is diagonal and $\bE'$ is off-diagonal. Then
$$z = \phi_{\alpha} ( (\bSigma' + \bE'  )x + \bA_0^{\dagger}\zeta )$$

where $$\nu := \| \bA_0^{\dagger}\zeta \|_{\infty} \le \frac{16 \| \zeta\|_2}{ \sigma_{\min} (\bAg)} \le \frac{16 \gamma_2 }{ \sigma_{\min} (\bAg)}.$$

For simplicity, we only focus on the expected update. The on-line version can be proved directly from this by setting a polynomially small $\eta$. The expected update is given by
$$\btA  = \bA + \eta \E [(\bAg x + \zeta)z^{\top} - \bA z z^{\top}].$$

Therefore,
\begin{eqnarray*}
\bAg (\btSigma + \btE) + \btN &=& \bA + \eta \E [(\bAg x + \zeta)z^{\top} - \bA z z^{\top}]
\\
&=& \bAg[ (\bSigma + \bE)(\bI - \eta\E[ z z^{\top}])  + \eta \E[x z^{\top}] ]    + \bNoise(\bI - \eta \E[z z^{\top}]) + \E[\zeta z^{\top}].
\end{eqnarray*}

So we still have
$$\btSigma + \btE = (\bSigma + \bE)(\bI - \eta \E[z z^{\top}])  + \eta \E[x z^{\top}] , \quad \btN =  \bNoise(\bI - \eta \E[ z z^{\top}]) + \E[\zeta z^{\top}].$$

By Lemma \ref{lem:noisy_decoding}, 
$$\btSigma + \btE = (\bSigma + \bE)(\bI - \eta \bSigma' \bDelta \bSigma') +\eta \bDelta  \bSigma'   + \bC_1$$
$$\btN = \bNoise(\bI - \eta \bSigma' \bDelta \bSigma') + \E[\zeta z^{\top}]+ \bNoise \bC_2.$$

where $\| \bC_1 \|_2, \|\bC_2\|_2 \le C_3$, and
$$C_3 =  (\nu + \beta)  \frac{kr}{D}+ \frac{m \ell^4 r^2}{\alpha^3 D^2} + \frac{\ell^2 \sqrt{km} r^{1.5}}{D^{1.5} \beta } + \frac{\ell^4 r^3 m}{\beta^2 D^2} + \frac{\ell^5 r^{2.5} m}{D^2 \alpha^2 \beta}.$$

\yingyu{need more details about applying Lemma \ref{lem:epoching} and Corollary~\ref{cor:epoching}}
\Ynote{I didn't mean  Lemma \ref{lem:epoching}, I mean proof of the main theorems, since they are just identical}
First, consider the update on $\btSigma + \btE$.
By a similar argument as in Lemma \ref{lem:epoching}, we know that as long as $C_3 = O\left(\frac{k}{D}\lambda \| \bE_0\|_2 \right)$ and $\nu = O(\ell)$, we can reduce the norm of $ \bE$ by a constant factor in polynomially many iterations. 
To satisfy the requirement on $C_3$,  we will choose $\alpha = 1/4, \beta = \frac{\lambda }{r}$. Then to make the terms in $C_3$ small, $m$ is set as follows. 
 \begin{enumerate}
 \item Second term: $$ m \le \frac{D k \lambda}{r^2 }.$$
  \item Third term: $$ m \le \frac{D \lambda^4 k }{ r^5 }.$$
   \item Fourth term: $$ m \le \frac{D \lambda^3 k }{ r^{5} }.$$
    \item Fifth term: $$ m \le \frac{D \lambda^2 k }{ r^{3.5} }.$$
 \end{enumerate}
This implies that after $\poly(\frac{1}{\veps})$ stages, the final $\bE$ will have 
$$\|\bE\|_2 = O\left( \frac{r \gamma_2}{\lambda \sigma_{\min}(\bAg)} + \veps \right).$$

Next, consider the update on $\btN$. 
Since the chosen value satisfies $C_3 \le \frac{1}{2}\sigma_{\min}(\bSigma' \bDelta \bSigma')$, we have
$$\|\btN \|_2 \le  \max\left\{\| \bNoise_0 \|_2, \frac{2 \| \E[\zeta z^{\top}] \|_2}{\sigma_{\min}(\bSigma' \bDelta \bSigma')} \right\}.$$

For the term $ \E[\zeta z^{\top}]$, we know that for every vectors $u, v$ with norm $1$,
$$u^{\top} \E[\zeta z^{\top}] v \le \E[ |\langle u, \zeta \rangle |  |\langle z, v \rangle |].$$

Since $z$ is non-negative, we might without loss of generality assume that $v$ is all non-negative, and obtain 
$$\E[ |\langle u, \zeta \rangle |  |\langle z, v \rangle |] \le \sqrt{\E[ \langle u, \zeta \rangle^2 ] \E[\langle z, v \rangle^2]} \le   \sqrt{\E[ \langle u, \zeta \rangle^2 ] }  \sqrt{\max_{i \in [D]} \E[z_i^2]} = O\left( \sqrt{\frac{k \gamma_1}{D}} \right).$$

Putting everything together and applying Corollary~\ref{cor:epoching} across stages complete the proof. 
\end{proof}

\begin{proof}[Proof of Theorem~\ref{thm:main_noise_simple}]
The theorem follows from Theorem~\ref{thm:main_noise} and noting that $\frac{\sqrt{ \gamma_1} }{  \lambda }  \sqrt{\frac{D}{k}}$ is smaller than $r \frac{\gamma_2 }{\lambda } \frac{\sigma_{\max} (\bAg)}{\sigma_{\min} (\bAg)}$ in order. 
\end{proof}

\section{Additional Experiments} \label{app:add_exp}

Here we provide additional experimental results. The first set of experiments in Section~\ref{app:robust_init} evaluates the performance of our algorithm in the presence of weak initialization, since for our theoretical analysis a warm start is crucial for the convergence. It turns out that our algorithm is not very sensitive to the warm start; even if there is a lot of noise in the initialization, it still produces reasonable results. This allows it to be used in a wide arrange of applications where a strong warm start is hard to achieve. 

The second set of experiments in Section~\ref{app:robust_sparsity} evaluates the performance of the algorithm when the weight $x$ has large sparsity. Note that our current bounds have a slightly strong dependency on the $\ell_1$ norm of $x$. We believe that this is only because we want to make our statement as general as possible, making only assumptions on the first two moments of $x$. If in addition, for example, $x$ is assumed to have nice third moments, then our bound can be greatly improved. 
Here we show that empirically, our algorithm indeed works for typical distributions with large sparsity. 

The final set of experiments in Section~\ref{app:qua_app} applies our algorithm on typical real world applications of NMF. In particular, we consider topic modeling on text data and component analysis for image data, and compare our method to popular existing methods.

\subsection{Robustness to Initializations} \label{app:robust_init}

In all the experiments in the main text, the initialization matrix $\bA_0$ is set to 
$
\bA_0 = \bAg(\bI + \bU)
$
where $\bI$ is the identity matrix and $\bU$ is a matrix whose entries are i.i.d. samples from the uniform distribution on $[-0.05,0.05]$.
Note that this is a very weak initialization, since $[\bA_0]^i = (1+ \bU_{i,i}) [\bAg]^i + \sum_{j \neq i} \bU_{j,i} [\bAg]^j$ and the magnitude of the noise component $\sum_{j \neq i} \bU_{j,i} [\bAg]^j$ can be larger than the signal part $(1+ \bU_{i,i}) [\bAg]^i$. 
  
Here, we further explore even worse initializations: 
$ 
 \bA_0 = \bAg (\bI + \bU) + \bNoise
$ 
where $\bI$ is the identity matrix, $\bU$ is a matrix whose entries are i.i.d. samples from the uniform distribution on $[-0.05,0.05] \times r_l$ for a scalar $r_l$, $\bNoise$ is an additive error matrix whose entries are i.i.d. samples from the uniform distribution on $[-0.05, 0.05] \times r_n$ for a scalar $r_n$. Here, we call $\bU$ the in-span noise and $\bNoise$ the out-of-span noise, since they introduce noise in or out of the span of $\bAg$. 

We varied the values of $r_l$ or $r_n$, and found that even when $\bU$ violates our assumptions strongly, or the column norm of $\bNoise$ becomes as large as the column norm of the signal $\bAg$, the algorithm can still recover the ground-truth up to small relative error. 
Figure~\ref{fig:robust_init}(a) shows the results for different values of $r_l$. Note that when $r_l = 1$, the in-span noise already violates our assumptions, but as shown in the figure, even when $r_l = 2$, the ground-truth can still be recovered, though at a slower yet exponential rate. 
Figure~\ref{fig:robust_init}(b) shows the results for different values of $r_n$. For these noise values, the column norm of the noise matrix $\bNoise$ is comparable or even larger than the column norm of the signal $\bAg$, but as shown in the figure, such noise merely affects on the convergence.

\begin{figure*}[!h]
\centering
\subfigure[Initialization with in-span noise]{\includegraphics[width=0.4\linewidth]{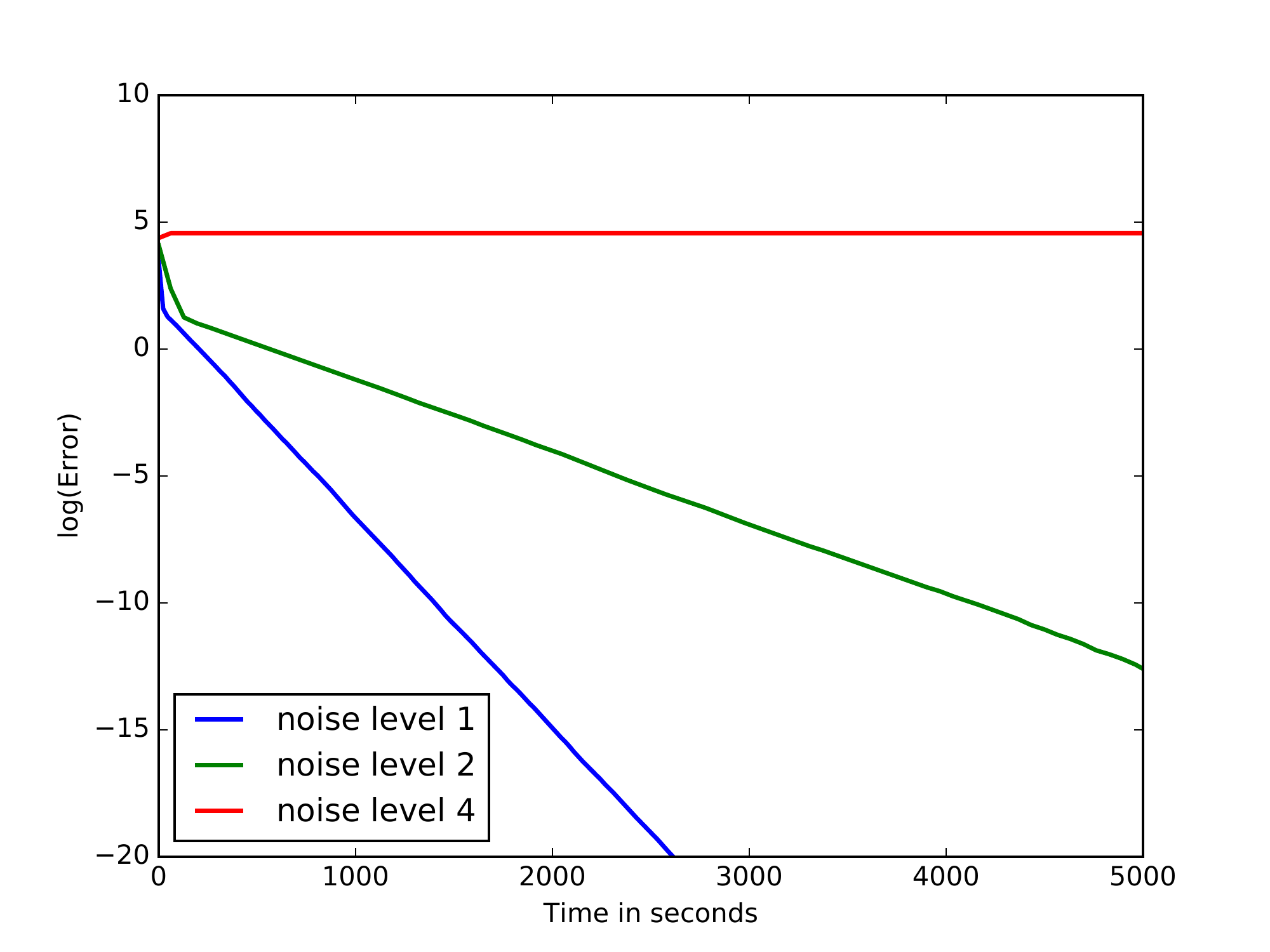}}
\subfigure[Initialization with out-of-span noise]{\includegraphics[width=0.4\linewidth]{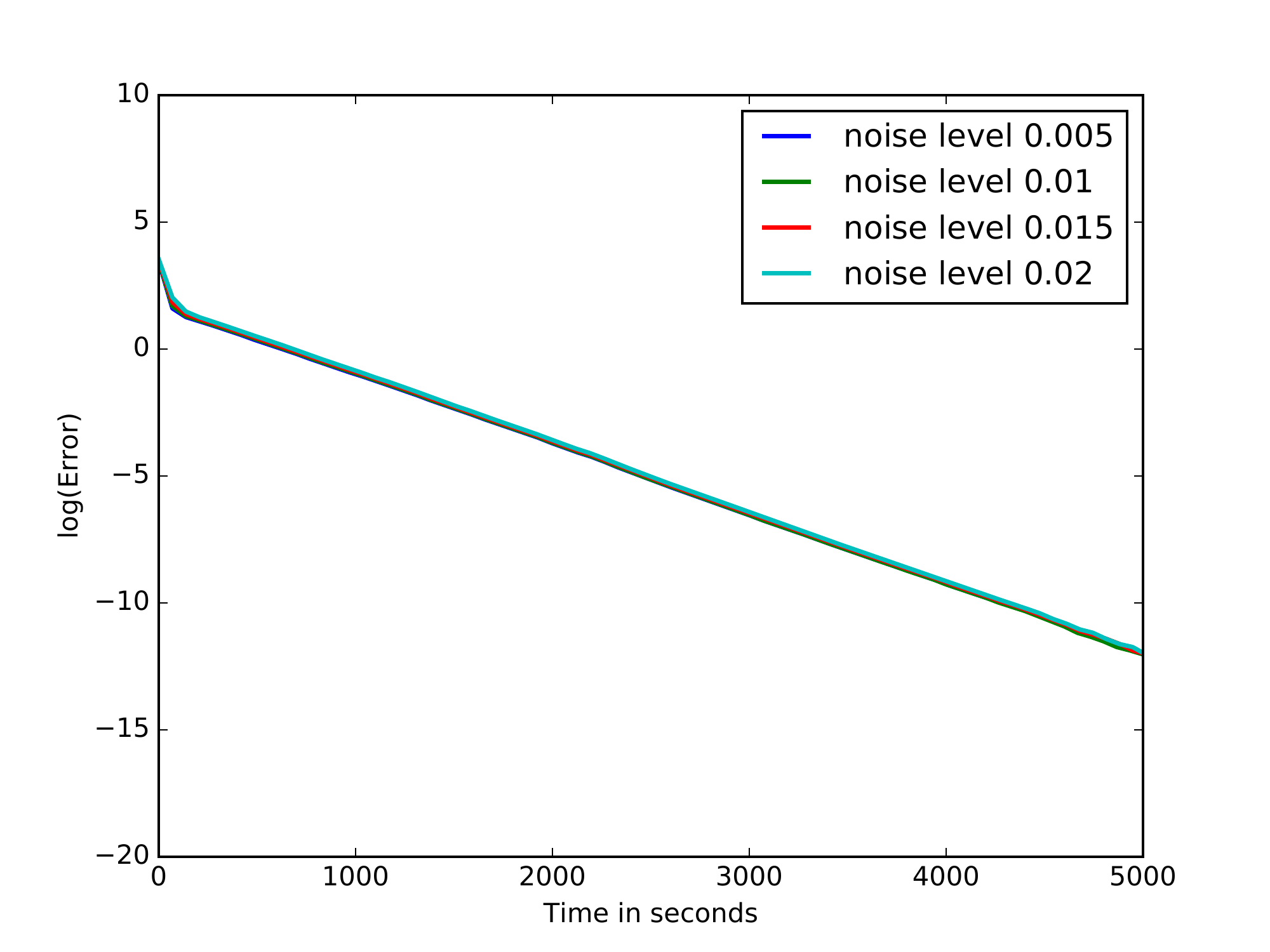}}
\caption{The performance of the algorithm \textsf{AND} with weak initialization. The $x$-axis is the running time (in seconds), the $y$-axis is the logarithm of the total correlation error. (a) Using different values for the noise level $r_l$ that controls the in-span noise in the initialization.
(b) Using different values for  the noise level $r_n$ that controls the out-of-span noise in the initialization.
}
\label{fig:robust_init}
\end{figure*}

\subsection{Robustness to Sparsity} \label{app:robust_sparsity}

We performed experiments on the \textsf{DIR} data with different sparsity. In particular, construct a $100 \times 5000$ matrix $\bX$, where each column is drawn from a Dirichlet prior $D(\mathbf{\alpha})$ on $d=100$ dimension, where $\mathbf{\alpha} = (\alpha/d, \alpha/d, \ldots, \alpha/d)$ for a scalar $\alpha$. Then the dataset is $\bY = \bAg \bX$. 	
We varied the $\alpha$ parameter of the prior to control the expected support sparsity, and ran the algorithm on the data generated. 

Figure~\ref{fig:robust_sparsity} shows the results. For $\alpha$ as large as $20$, the algorithm still converges to the ground-truth in exponential rate. When $\alpha=80$ meaning that the weight vectors (columns in $\bX$) have almost full support, the algorithm still produces good results, stabilizing to a small relative error at the end. This demonstates that the algorithm is not sensitive to the support sparsity of the data. 

\begin{figure*}[!h]
\centering
\includegraphics[width=0.6\linewidth]{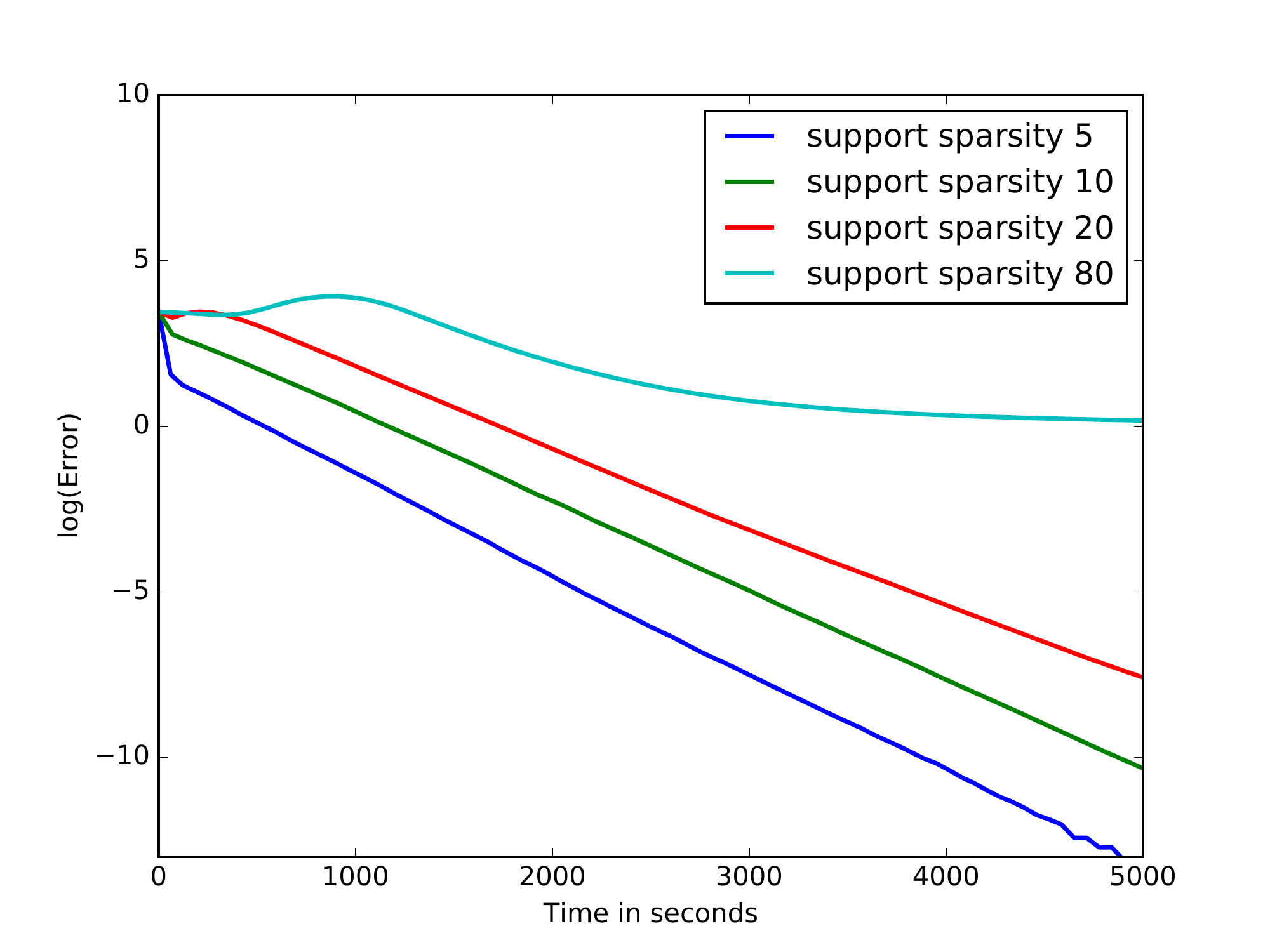}
\caption{The performance of the algorithm \textsf{AND} on data generated from Dirichlet prior on $x$ with different sparsities. The $x$-axis is the running time (in seconds), the $y$-axis is the logarithm of the total correlation error. 
}
\label{fig:robust_sparsity}
\end{figure*}

\subsection{Qualitative Results on Some Real World Applications} \label{app:qua_app}

We applied our algorithm to two popular applications with real world data to demonstrate the applicability of the method to real world scenarios. Note that the evaluations here are qualitative, due to that the guarantees for our algorithm is the convergence to the ground-truth, while there are no predefined ground-truth for these datasets in practice. Quantitative studies using other criteria computable in practice are left for future work. 

\subsubsection{Topic Modeling}

Here our method is used to compute $10$ topics on the 20newsgroups dataset, which is a standard dataset for the topic modeling setting. 
Our algorithm is initialized with $10$ random documents from the dataset, and the hyperparameters like learning rate are from the experiments in the main text. Note that better initialization is possible, while here we keep things simple to demonstrate the power of the method.

Table~\ref{tab:topic} shows the results of the NMF method and the LDA method in the sklearn package,\footnote{\url{http://scikit-learn.org/}} and the result of our \textsf{AND} method. It shows that our method indeed leads to reasonable topics, with quality comparable to well implemented popular methods tuned to this task.

\begin{table}
	\centering
		\begin{tabular}{c| l }
		\hline 
		Method & Topic \\
\hline
\hline
\multirow{ 10 }{*}{NMF (sklearn)}  & 
			
just people don think like know time good make way \\ & 

windows use dos using window program os drivers application help \\ & 

god jesus bible faith christian christ christians does heaven sin \\ & 

thanks know does mail advance hi info interested email anybody\\ & 

car cars tires miles 00 new engine insurance price condition \\ & 

edu soon com send university internet mit ftp mail cc \\ & 

file problem files format win sound ftp pub read save \\ & 

game team games year win play season players nhl runs \\ & 

drive drives hard disk floppy software card mac computer power \\ & 

key chip clipper keys encryption government public use secure enforcement  \\
\hline
\hline
\multirow{ 10 }{*}{LDA (sklearn)}  & 

edu com mail send graphics ftp pub available contact university \\ & 

don like just know think ve way use right good \\ & 

christian think atheism faith pittsburgh new bible radio games \\ & 

drive disk windows thanks use card drives hard version pc \\ & 

hiv health aids disease april medical care research 1993 light \\ & 

god people does just good don jesus say israel way \\ & 

55 10 11 18 15 team game 19 period play \\ & 

car year just cars new engine like bike good oil \\ & 

people said did just didn know time like went think\\ & 

key space law government public use encryption earth section security \\ 

\hline
\hline

\multirow{ 10 }{*}{AND (ours)}  & 

game team year games win play season players 10 nhl \\ & 

god jesus does bible faith christian christ new christians 00 \\ & 

car new bike just 00 like cars power price engine\\ & 

key government chip clipper encryption keys use law public people \\ & 

young encrypted exactly evidence events especially error eric equipment entire\\ & 

thanks know does advance mail hi like info interested anybody \\ & 

windows file just don think use problem like files know \\ & 

drive drives hard card disk software floppy think mac power  \\ & 

edu com soon send think mail ftp university internet information \\ & 

think don just people like know win game sure edu \\ 

\hline

		\end{tabular}
	\caption{Results of different methods computing $10$ topics on the 20newsgroups dataset. Each topic is visualized by using its top frequent words, and each line presents one topic.}
	\label{tab:topic}
\end{table}

\subsubsection{Image Decomposition}

Here our method is used to compute $6$ components on the Olivetti faces dataset, which is a standard dataset for image decomposition. 
Our algorithm is initialized with $6$ random images from the dataset, and the hyperparameters like learning rate are from the experiments in the main text. Again, note that better initialization is possible, while here we keep things simple to demonstrate the power of the method.

Figure~\ref{fig:face} shows some examples from the dataset, the result of our \textsf{AND} method, and 6 other methods using the implementation in the sklearn package. 
It can be observed that our method can produce meaningful component images, and the non-negative matrix factorization implementation from sklearn produces component images of similar quality. The results of these two methods are generally better than those by the other methods.

\newcommand{\facescale}{0.35}
\begin{figure*}[!h]
\centering
\subfigure[Example images in the dataset]{\includegraphics[width=\facescale\linewidth]{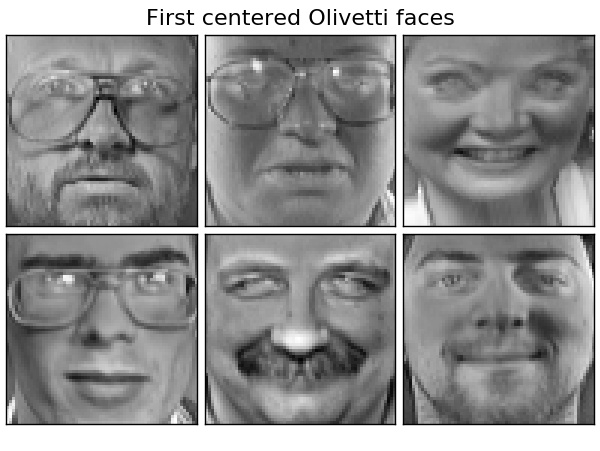}}
\subfigure[Our \textsf{AND} algorithm]{\includegraphics[width=\facescale\linewidth]{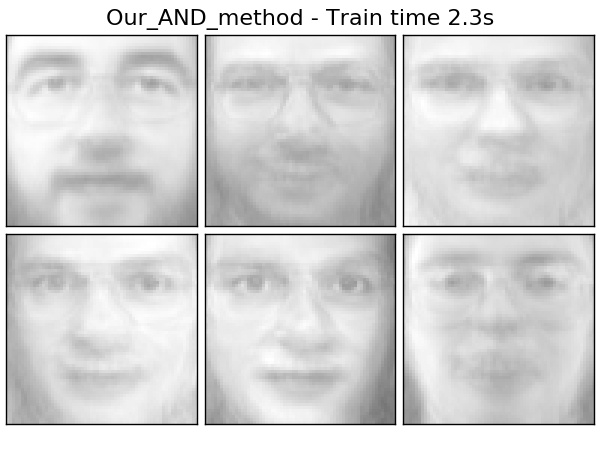}}
\subfigure[K-means]{\includegraphics[width=\facescale\linewidth]{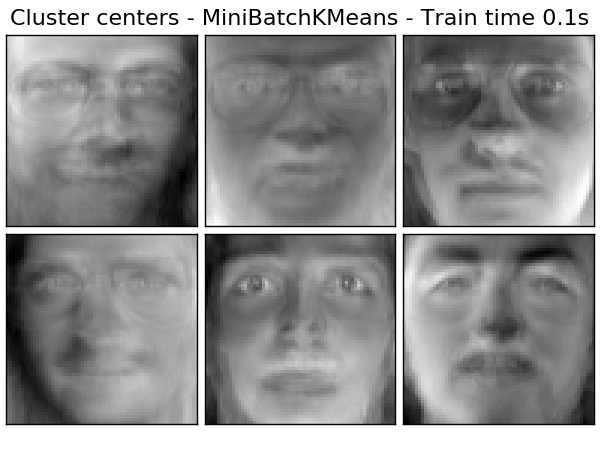}} 
\subfigure[Principal Component Analysis]{\includegraphics[width=\facescale\linewidth]{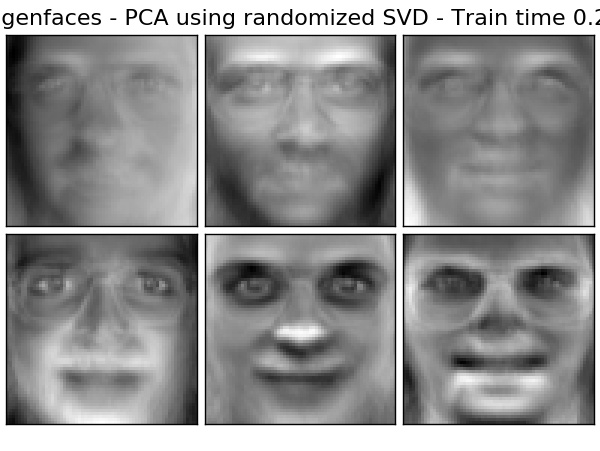}} 
\subfigure[Independent Component Analysis]{\includegraphics[width=\facescale\linewidth]{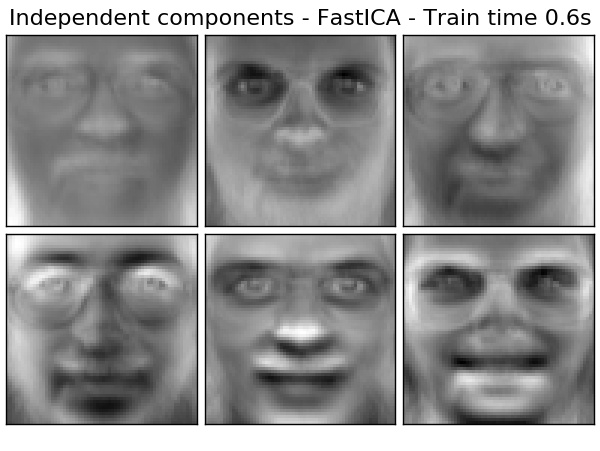}} 
\subfigure[Dictionary learning]{\includegraphics[width=\facescale\linewidth]{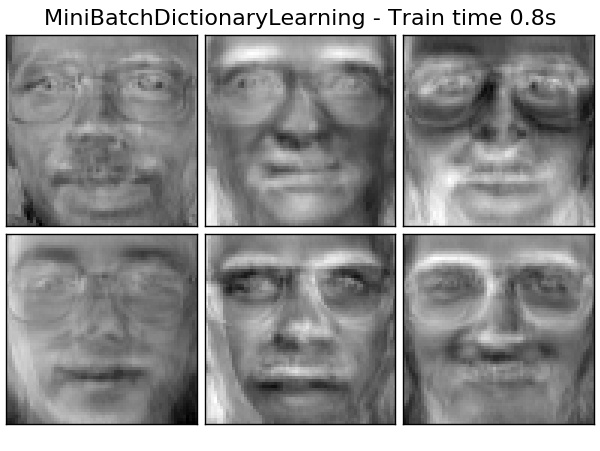}} 
\subfigure[Non-negative matrix factorization (sklearn)]{\includegraphics[width=\facescale\linewidth]{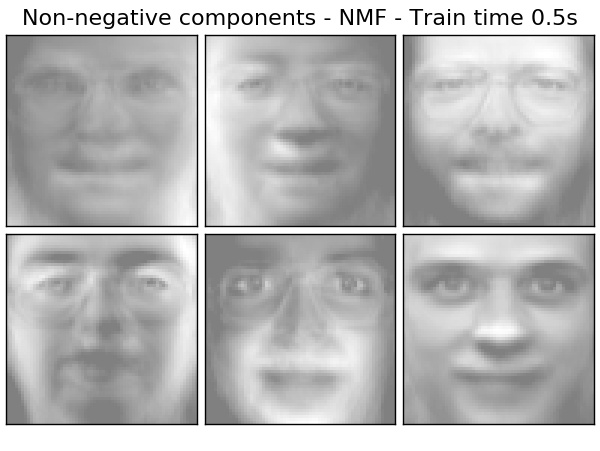}} 
\subfigure[Sparse Principal Component Analysis]{\includegraphics[width=\facescale\linewidth]{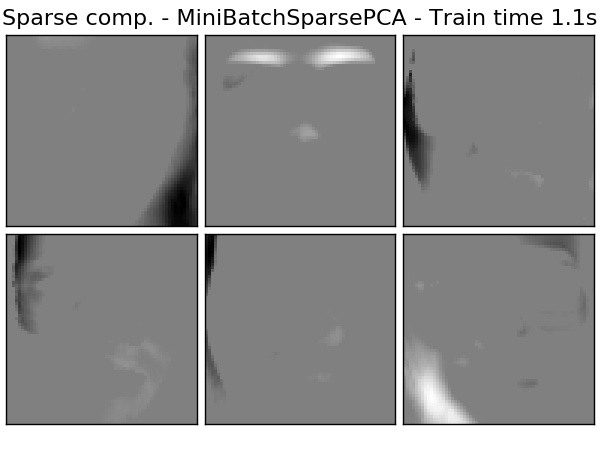}}
\caption{The results of different methods computing $6$ components on the Olivetti faces dataset. For all the competitors, we used the implementations in the sklearn package. 
}
\label{fig:face}
\end{figure*}

\end{document}